\documentclass[11pt, letterpaper]{article}
\usepackage{fullpage}
\usepackage[T1]{fontenc}
\usepackage{amsmath,amsfonts,amsthm,amssymb,mathtools}
\usepackage[figure,boxed,lined]{algorithm2e}
\usepackage{algorithmic}

\usepackage[pagebackref]{hyperref}
\usepackage[dvipsnames]{xcolor}
\hypersetup{
colorlinks=true,
urlcolor=NavyBlue,
linkcolor=NavyBlue,
citecolor=ForestGreen
}

\usepackage{color}
\usepackage{multirow}
\usepackage{enumerate}
\usepackage{graphicx}
\usepackage{comment}

\usepackage{natbib}

\usepackage{thmtools}
\usepackage{thm-restate}

\usepackage{multirow}
\usepackage{slashbox}

\usepackage{microtype}
\usepackage{subfig}
\usepackage{booktabs}

\theoremstyle{plain}
\newtheorem{thm}{\protect\theoremname}
\theoremstyle{definition}
\newtheorem{defn}[thm]{\protect\definitionname}
\theoremstyle{plain}
\newtheorem{lem}[thm]{\protect\lemmaname}
\theoremstyle{plain}
\newtheorem{cor}[thm]{\protect\corollaryname}

\makeatother

\providecommand{\corollaryname}{Corollary}
\providecommand{\definitionname}{Definition}
\providecommand{\lemmaname}{Lemma}
\providecommand{\theoremname}{Theorem}

\global\long\def\grad#1{\boldsymbol{\mathit{g}}\left(#1\right)}\global\long\def\vb{\boldsymbol{b}}\global\long\def\va{\boldsymbol{\mathit{a}}}\global\long\def\vdelta{\boldsymbol{\mathit{\Delta}}}\global\long\def\v{\boldsymbol{\mathit{v}}}\global\long\def\xx{\x}\global\long\def\x{\boldsymbol{\mathit{x}}}\global\long\def\y{\boldsymbol{\mathit{y}}}\global\long\def\xp{\boldsymbol{\mathit{x'}}}\global\long\def\xstar{\boldsymbol{\mathit{x}}^{\boldsymbol{\mathit{\star}}}}\global\long\def\grid{\delta}\global\long\def\gridstar{\grid_{\star}}\global\long\def\gridgrad{\grid_{\nabla}}\global\long\def\qg{{\bf Q}}\global\long\def\qf{{\bf Q}^{\bf g}}\global\long\def\qs{{\bf Q}^{\bf w}}\global\long\def\unif#1{\text{Unif}\left(#1\right)}\global\long\def\fstar{f^{\star}}\global\long\def\err{\varepsilon}\global\long\def\interval{\left[-1/2,1/2\right)}\global\long\def\intg{\left[-\grid/2,\grid/2\right)}\global\long\def\intgs{\left[-\gridstar/2,\gridstar/2\right)}\global\long\def\one{{\bf 1}}

\global\long\def\gradnorm{G_{\ell_1}}

\global\long\def\sigmagrad{\sigma_{\nabla}}

\global\long\def\dom{\mathcal{G}}

\renewcommand{\paragraph}[1]{\noindent\vspace{0.2em}\textbf{#1}}
 \theoremstyle{plain}

\theoremstyle{definition}

\theoremstyle{remark}

\usepackage[textsize=tiny]{todonotes}

\title{Quantized Distributed Training of Large Models with Convergence Guarantees}

\author{Ilia Markov\footnote{Institute of Science and Technology Austria, \texttt{ilia.markov@ist.ac.at}} \and
Adrian Vladu\footnote{CNRS \& IRIF, Université Paris Cité, \texttt{vladu@irif.fr}} \and
Qi Guo\footnote{Max Planck Institute for Informatics, \texttt{qiguo@mpi-inf.mpg.de}} \and
Dan Alistarh\footnote{Institute of Science and Technology Austria, \texttt{dan.alistarh@ist.ac.at}}}

\date{}
\begin{document}

\maketitle

\begin{abstract}
Communication-reduction techniques are a popular way to improve scalability in data-parallel training of deep neural networks (DNNs). The recent emergence of large language models such as GPT has created the need for new approaches to exploit data-parallelism. Among these, fully-sharded data parallel (FSDP) training is highly popular, yet it still encounters scalability bottlenecks. One reason is that applying compression techniques to FSDP is challenging: as the vast majority of the communication involves the model's weights, direct compression alters convergence and leads to accuracy loss. We present QSDP, a variant of FSDP which supports both gradient and weight quantization with theoretical guarantees, is simple to implement and has essentially no overheads. To derive QSDP we prove that a natural modification of SGD achieves convergence even when we only maintain quantized weights, and thus the domain over which we train consists of quantized points and is, therefore, highly non-convex. We validate this approach by training GPT-family models with up to 1.3 billion parameters on a multi-node cluster. Experiments show that QSDP preserves model accuracy, while completely removing the communication bottlenecks of FSDP, providing end-to-end speedups of up to 2.2x. 
\end{abstract}

\section{Introduction}
\label{introduction}

The impressive recent progress of Deep Learning in tasks such as natural language processing 
and computer vision has been accompanied by massive increases in parameter counts. 
For instance, large language models (LLMs) from Transformer family, such as GPT~\citep{radford2018improving}, OPT~\citep{zhang2022opt} and BLOOM~\citep{BLOOM} 
easily count billions of trainable parameters, which induces tremendous computational and memory costs. 
Training such models can easily exceed the memory capacity of a single computational unit, such as a GPU. 

As a consequence, standard distribution strategies such as \emph{data-parallel training}~\cite{bottou2010large}, which require each node to be able to keep all parameters in memory, are no longer directly applicable. 
Several novel distribution strategies have been proposed to mitigate this challenge, such as \emph{model-parallel training}~\cite{shoeybi2019megatron, raffel2020exploring}, \emph{pipeline-parallel training}~\cite{huang2019pipe,harlap2018pipedream} and \emph{model sharding}~\cite{ren2021zero,rajbhandari2020zero,rasley2020deepspeed, FairScale2021}.  

We consider the communication costs of distribution strategies for massive models, and focus on \emph{Fully-Sharded Data-Parallel (FSDP)} distributed training, which is among the most popular and user-friendly approaches to mitigate per-node memory limitations. 
FSDP is supported natively by Pytorch~\cite{pytorch}, Facebook \texttt{fairscale}~\cite{FairScale2021}, and Microsoft DeepSpeed~\cite{ren2021zero}, where it is known as ZeRO-3.  

The main idea behind FSDP is that both the training data \emph{and the model parameters} are partitioned among the $P$ nodes. That is, only a $1/P$ partition of the parameters of each layer is stored at a node. 
Then, both for the forward and for the backward pass, nodes proceed synchronously layer-by-layer, 
gathering full weights for the current layer, via \emph{all-to-all communication}, before executing its forward or backward operation. 
After this operation is complete, nodes can discard the current layer's received weights partitions, and move to the next layer.  
(Please see Figure~\ref{fig:fsdp_scheme} for an illustration, and Section~\ref{sec:background} for a detailed description.) 

The key advantage of this pattern is that it reduces memory usage linearly in $P$. 
Thus, it enables running models with billions of parameters on small or medium-sized clusters~\cite{FairScale2021, MosaicmlBlog}. 
At the same time, FSDP faces challenges in terms of \emph{communication efficiency}: 
since every forward and backward pass relies on all-to-all weight exchanges, FSDP can put massive pressure on the network bandwidth, which becomes a bottleneck.

As we will show, all-to-all communication leads to significant communication bottlenecks when training LLMs on multi-node clusters. 
Two key challenges to removing this communication bottleneck are that 
1) a majority of FSDP's communication are \emph{layer weights}: 
quantizing them naively loses theoretical convergence, and can easily lead to practical divergence; 
2) the FSDP setting poses stringent compute and memory constraints, restricting the set of approaches. 

\paragraph{Contribution.} 
We propose the first communication-efficient variant of FSDP, called QSDP,
which provides both convergence guarantees, and strong practical performance. 
QSDP is inspired by on a new analysis of SGD convergence with \emph{full quantization of transmitted model state}. 
That is, we show that a simple modified variant of SGD can allow both weights and gradients to be quantized during training, 
without additional per-node memory, nor costly local computation.
We find the fact that this is possible with convergence guarantees surprising, 
since nodes only observe \emph{biased estimators} of the gradients, taken over quantized weights, without any error-correction~\cite{karimireddy2019error}.  
From the practical perspective, our approach is fast and easy to implement, 
and completely removes the communication bottlenecks of FSDP, while recovering accuracy for billion-parameter GPT models. 

At a high level, the QSDP algorithm simply performs weight and gradient quantization before the corresponding FSDP all-to-all communication steps. 
While gradient compression can be performed using standard unbiased compressors, e.g.~\cite{alistarh2017qsgd}, 
weight compression is performed using a carefully-designed unbiased estimator.
Our key contribution is in the analysis: we model the training process as a new instance of sparse recovery~\cite{blumensath2008iterative, foucart2012sparse}, in which 1) the projection step is performed via quantization and not sparsification, and 2) the gradient step is itself quantized. 
This connection allows us to prove, under analytic assumptions, that QSDP converges towards a minimizer of the loss over the set of lattice points corresponding to the quantization being employed. We believe this is the first instance of such an analysis. 

We complement our analysis with an efficient implementation of QSDP in Pytorch~\cite{pytorch}, 
which we validate by training LLMs from the GPT family~\cite{radford2018improving, zhang2022opt} between 125M and 1.3B parameters, 
on a multi-node multi-GPU environment on Amazon EC2. 
Our experiments first show that communication bottlenecks can significantly impact standard FSDP in this standard practical setting, 
and that QSDP essentially removes such bottlenecks, without impact on accuracy. 
For example, QSDP can train GPT-1.3B to essentially the same perplexity up to 2.2x faster than standard FSDP on a 10Gbps network.  
In addition, we also introduce a ``learned'' adaptive weight quantization approach which can further reduce bit-width, without significant accuracy impact.

\begin{figure*}[h]
    \centering
    \includegraphics[width=\textwidth]{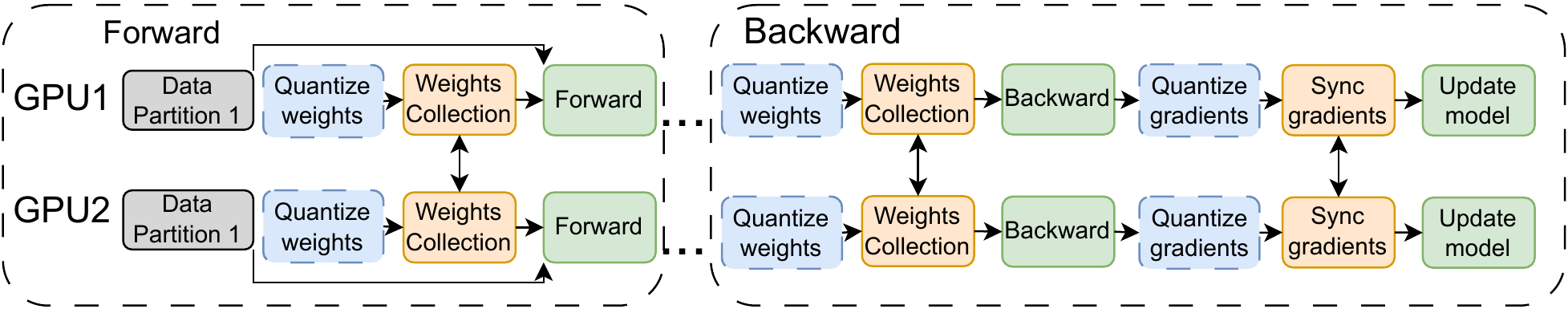}
    \caption{Scheme of (Quantized) Fully Sharded Data Parallel algorithm. During forward pass we collect the missing partitions of layer's weights, compute its activations and discard the partitions. At backward pass, we collect the weights again, compute the gradients, synchronize the gradients corresponding to our partition.}
    \label{fig:fsdp_scheme}
\end{figure*}

\vspace{-0.5em}
\section{Related Work}
\vspace{-0.5em}
\label{sec:related-work}

Over the past decade, there has been a massive amount of work on communication-efficient variants of Data-Parallel SGD, e.g.~\cite{seide2014, dryden2016communication, alistarh2017qsgd, vogels2019powersgd,tang2019doublesqueeze, wang2018atomo}. (Please see~\cite{ben2019demystifying} for a survey.) 
The vast majority of this work focuses on gradient compression, the main communication cost of SGD, and is thus mostly orthogonal to our work. 
The massive scale of recent deep models, e.g.~\cite{chowdhery2022palm, brown2020language} has led to significant work on novel distribution strategies~\cite{ren2021zero,rajbhandari2020zero,rasley2020deepspeed, FairScale2021} adapted to the requirements of these models, 
among which FSDP is a standard approach, e.g.~\cite{chowdhery2022palm}. 
While there is recent work on further optimizing FSDP, e.g.~\cite{jiang2022osdp, miao2022galvatron}, 
we are the first to investigate and address its communication costs. 
Our results are part of a broader line of work using different techniques to make the training of massive models amenable to standard infrastructure, e.g.~\cite{wang2022fine, yuan2022decentralized, borzunov2022petals}. 

Quantized weight exchange during training has been investigated independently in the context of decentralized distributed learning. 
\citet{tang2018decentralization} presents a scheme which supports quantized weight exchange by having each node extrapolate each of its neighbors' model values; yet, this would require unrealistic $\Theta(Pd)$ extra memory in our case. 
Similarly, other work in this vein~\citep{koloskova2019decentralized, nadiradze2021asynchronous, lu2020moniqua} either requires additional storage, or would not fit the FSDP algorithm structure. Both our analysis approach and our algorithms' guarantees are different relative to this line of work. 

Recently, there has been a surge of interest in \emph{post-training} quantization approaches for large language models, which reduce the deployment costs of already trained models~\cite{yao2022zeroquant, dettmers2022llm, frantar2022gptq, xiao2022smoothquant}. Our work is complementary, in the sense that we show that quantized weights and gradient representations can be applied \emph{during training}, without accuracy loss, leading to training speedup. On the other hand, these post-training approaches would be too expensive to be executed for compression at training time. 

A parallel line of work aims to perform \emph{fully-quantized} training of DNNs~\cite{banner2018scalable, zhu2020towards}. 
One general finding from this line of work is that integrating weight and gradient quantization \emph{into training} is extremely challenging, even when using 8bit precision, from both accuracy and performance perspectives. Specifically, this line of work investigates model modifications via e.g. parameter tuning and specialized normalization layers, in order to recover accuracy. By contrast, we preserve model structure, and do not modify hyper-parameter values, although we only quantize transmitted state. 

\section{Background and Motivation}
\label{sec:fsdp}

\subsection{Data-Parallel Training} \label{sec:dptraining}
In this classic SGD distribution pattern~\cite{bottou2010large}, each node (e.g., GPU) holds a copy of the model, and the 
data is partitioned among the nodes. 
Each training step samples a subset of the data called a \emph{batch}, 
performs a \emph{forward pass} over the batch to obtain model predictions, 
and then performs a \emph{backward pass} to obtain gradient updates. 
Finally, nodes communicate their local gradient updates in all-to-all fashion to keep the model in sync.

\subsection{Gradient Compression}

Transmitting gradients is the key communication cost of Data-Parallel SGD, and there has been a tremendous amount of work on addressing the resulting bandwidth bottleneck~\cite{seide2014, dryden2016communication, strom2015scalable}. (As this area is extremely vast, we refer to~\citet{ben2019demystifying, liu2020distributed} for a full overview.) 
Of these, gradient quantization is a particularly-popular technique, which has the advantage that variants of it can be implemented without additional memory cost. 
A simple example is the QSGD technique~\cite{alistarh2017qsgd}, which is essentially a codebook compression method which maps each gradient value 
to a point on a uniform grid, via randomized rounding. 
For this, values are first scaled to the range $[-1, 1]$, and then each scaled coordinate $v$ is mapped to one of the endpoints of its quantization interval $v \in [q_i, q_{i+1}]$ via the following rule:
$$
q(v) = 
\begin{cases}
q_i, \text{with probability } \frac{ v - q_{i} }{q_{i + 1} - q_i}, \\
q_{i+1}, \text{otherwise.}
\end{cases}
$$

It is easy to see that this gradient estimator is unbiased with respect to the stochastic quantization, 
and that its variance can be bounded by the norm of the original gradient. 
We will revisit this scheme in Sections~\ref{sec:analysis} and~\ref{sec:implementation}.

\subsection{Fully-Sharded Data-Parallel Training}

As the name suggests, FSDP starts from the Data-Parallel (DP) approach. The main observation is that nodes do not necessarily need to store the full set of parameters at every stage of training, in particular during the backward pass. Specifically, we use the scarce GPU memory to represent only those layers which are in the forward-backward ``working set'' at a given moment of time. 

Initially, model parameters are partitioned, so that each of the $P$ workers is assigned a distinct $1/P$ fraction 
of each layer's weights. At each optimization step (see Figure~\ref{fig:fsdp_scheme}, ignoring the dashed quantization operations), 
before the forward pass on a layer, each worker collects the missing partitions from other workers, computes the output activations, 
discards the received partitions and proceeds to the next layer. 
For the backward pass, workers again collect all layer weights, compute the gradients, synchronize them, discard the layer weights and proceed to the next layer. Technically, each optimization step consists of two AllGather collective operations for weights, and one Reduce-Scatter to sync gradients
(full pseudocode in Appendix~\ref{app:training_details}).

One can easily check that the above approach implements the standard SGD iteration one-to-one, relative to a sequential execution. 
If we denote by $\y_t$ the model's parameter vector used at iteration $t$, and by $\grad{\y_t}$ the average of the nodes' stochastic gradients at step $t$, taken at $\y_t$, then, for learning rate $\eta$, we can model the iteration as 
\begin{equation} 
\y_{t+1}= \y_t - \eta \grad{\y_t}. \label{eq:sgd}
\end{equation}
 
\paragraph{FSDP with Compression.} 
The natural way to reduce the cost of weight and gradient transmission in the above scheme would be to simply quantize them before transmission. 
(Please see the full Figure~\ref{fig:fsdp_scheme}.) 
To examine the impact of adding compression on the above SGD iteration, let us consider abstract quantization operators $\qs$ applied to the weights, and $\qf$ applied to the gradients. (We will specify these quantization functions precisely in Section~\ref{sec:analysis}, and the exact implementation in Section~\ref{sec:implementation}.)
For  iteration $t \geq 0$, let $\v_t$ be a ``virtual'' view of the model weights at the beginning of iteration $t$, obtained by aggregating 
all the weights, across all the weight partitions, \emph{in full precision}. 

First, notice that, if we apply $\qs$ before all transmissions, then the algorithm will only observe the \emph{quantized} version of $\v_t$, 
which we denote by $\qs(\v_t)$. 
Then, we can re-write one iteration of the algorithm as 
\begin{equation*}
    \v_{t+1} = \qs( \v_t ) - \eta \qf( \grad{\qs(\v_t)} ).  
\end{equation*}

This formulation inspires the notation $\x_t = \qs(\v_t)$, as the algorithm only ``sees'' the quantized version of the full-precision weights.  
Then, we get the following iteration:
\begin{equation}
    \x_{t+1} = \qs( \x_t - \eta \qf( \grad{ \x_t })), 
\end{equation}

\noindent which would correspond to an abstractly-quantized version of FSDP. 
This iteration is the starting point for our analysis in the next section.

\section{SGD with Quantized Weights and Provable Convergence}
The cornerstone of our method consists of a stochastic gradient method that provably converges to a good \emph{quantized} iterate, under reasonable analytic assumptions. One main novelty is that it converges despite the fact that the domain is \emph{non-convex}. At a very high level, it is similar to the iterative hard thresholding (IHT) method, which  achieves provable guarantees despite the fact that it seeks a good iterate in the set of vectors of bounded sparsity~\cite{blumensath2008iterative}.
Throughout this section, we abstract away specifics of the system architecture, since they are not relevant to our analysis. We explain their relationship to the actual implementation in Section~\ref{sec:implementation}.

\subsection{Background and Assumptions}\label{sec:background}

The main challenge we face is to obtain quantized solutions to optimization problems that seek to minimize a function $f:\mathbb{R}^n \rightarrow \mathbb{R}$:
\begin{equation}
\min_{\x \in \dom} f(\x)\,,
\end{equation}
where the domain $\dom$ is a lattice that allows for an efficient communication of its elements. We restrict our attention to shifts of the  lattice $\grid \mathbb{Z}^n$ along the direction of the all-ones vector. Formally, $\dom = \{\grid \mathbb{Z}^n + r \one: r\in \intg \}$.

\paragraph{Overview.} Even in the case where $f$ is convex, the non-convex structure of $\dom$ makes it incredibly difficult to obtain a good minimizer to $f$ without suffering a large loss. In fact, problems of this form are generally NP-hard. However, we show that when $f$ is reasonably well-conditioned we can obtain strong convergence guarantees. The idea consists of alternating stochastic gradient descent steps with applications of a quantization operator $\qs$ which projects the new iterate back onto a certain subset of $\dom$. 
Letting $\grad{\x_t}$ be a stochastic gradient, and $\grid$ a parameter which determines the coarseness of the quantization grid that we project onto,
our update at step $t+1$ has the following form:
\begin{equation}
\x_{t+1}=\qs_{\grid}\left(\x_t - \eta \grad{\x_t} \right)\,.\label{eq:iteration}
\end{equation}

This formulation covers the practical case where the stochastic gradient $\grad{\x_t}$ corresponds to a mini-batch stochastic gradient.  Indeed, as in practice $f$ takes the form $f(\x) = \frac{1}{P m} \sum_{i=1}^P \sum_{j=1}^m f(\x;\y_j)$, where $S = \{\y_1 , \dots , \y_m \}$ are data samples, and $f_i(\x)$ are loss functions at individual nodes,  the stochastic gradients obtained via backpropagation takes the form $\frac{1}{\vert B \vert} \sum_{j\in B} \nabla f_i(\x; \y_j)$, where $i$ is a random node, and $B$ is a sampled mini-batch.

\paragraph{Quantization by Random Shift.} For weight quantization, we consider the following unbiased stochastic quantization method. 
To quantize a vector, we first sample a single fixed random scalar $r$, then shift all the coordinates of the vector by $r$. 
For vector encoding, it rounds each coordinate to the nearest neighbor on the quantization grid, and sends the lattice coordinates of the resulting vector, together with the scalar $r$. For decoding, it takes the lattice point, and undoes the shift on the quantized coordinates. 
The notable difference between this and more standard quantization methods, e.g.\cite{alistarh2017qsgd}, is that quantization is \emph{dependent} across coordinates. In exchange for losing independence, it allows us to provide stronger guarantees in the context of weight quantization. 
We define it formally:

\begin{defn}
[quantization by random shift]\label{def:qshift} Let $\grid>0$
be a scalar defining the coarseness of the quantization grid. Let
a scalar $r\in\intg$, and let the deterministic operator $\qs_{r,\grid}:\mathbb{R}\rightarrow\mathbb{R}$
which rounds to the nearest element in $\grid\mathbb{Z}+r$:
\[
\qs_{r,\grid}\left(x\right)=\grid\cdot\left\lfloor \frac{x-r}{\grid}\right\rceil +r\,.
\]
Define the randomized quantization operator $\qs_{\grid}:\mathbb{R}\rightarrow\mathbb{R}$
via $\qs_{\grid}\left(x\right)=\qs_{r,\grid}\left(x\right)$, for
a random $r\sim\unif{\intg}$. We apply $\qs_{\grid}$ to vectors,
with the meaning that it is dependently applied to each coordinate
for a single random shift $r$.
\end{defn}

We use this quantization methods to show that, for a well conditioned loss function, and appropriate grid parameters, the iteration (\ref{eq:iteration}) converges, under reasonable analytical assumptions, to a set of weights
that are comparable in quality to the best possible quantized weights from a slightly coarser grid. We note that to further reduce communication costs, our method also supports gradient quantization in addition to weight quantization, provided that gradients are quantized using an (arbitrary) unbiased estimator.

\paragraph{Analytical Assumptions.}
Formally, our analysis uses the following assumptions on $f$.
\begin{enumerate}
\item Unbiased gradient estimators with variance $\sigma$: $\mathbb{E}\left[\grad{\x}\vert\x\right]=\nabla f\left(\x\right)$.
\item For $\beta>0$, the $\beta$-smoothness condition: for all \ensuremath{\x,\vdelta},
\[
f\left(\x+\vdelta\right)\leq f\left(\x\right)+\left\langle \nabla f\left(\x\right),\vdelta\right\rangle +\frac{\beta}{2}\left\Vert \vdelta\right\Vert _{2}^{2}\,.
\]
\item For $\alpha>0$, the Polyak-\L ojasiewicz ($\alpha$-PL) condition:
\[
\frac{1}{2}\left\Vert \nabla f\left(\x\right)\right\Vert _{2}^{2}\geq\alpha\left(f\left(\x\right)-f^{*}\right)\,,
\]
where $f^{*}=\min_{\x}f\left(\x\right)$.
\end{enumerate}
The first two assumptions are standard in stochastic optimization (see e.g.~\cite{lin2019dynamic}). The Polyak-Łojasiewicz (PL) condition~\citep{karimi2016linear}  is common in non-convex optimization, and versions of it are essential in the analysis of DNN training~\citep{liu2020toward, allen2019convergence}.  In words, it states that small gradient norm, i.e. approximate stationarity, implies closeness to optimum in function value. 

\subsection{Main Theoretical Results}

We are now ready to state our main analytical result.

\begin{restatable}{thm}{mainthm}
\label{thm:mainthm}
Let $\alpha, \beta, \gridstar, \err > 0$ and $\sigma \geq 0$ be real parameters, and let 
$\eta=\min\left\{ \frac{3}{10}\frac{\err\alpha}{\sigma^{2}},1\right\}$. 
Let $f:\mathbb{R}^{n}\rightarrow\mathbb{R}$ be a
$\beta$-smooth and $\alpha$-PL function, with access to a stochastic
gradient $\grad{\x}$, i.e. $\mathbb{E}\left[\grad{\x}\vert\x\right]=\nabla f\left(\x\right)$
with bounded variance $\mathbb{E}\left\Vert \grad{\x}-\nabla f\left(\x\right)\right\Vert _{2}^{2}\leq\sigma^{2}$. For each $r\in\intgs$, let $\xstar_{r,\gridstar}$
be any minimizer of $f$ over $\gridstar\mathbb{Z}^{n}+r\one$.
Let $\grid=\frac{\eta}{\left\lceil 16\left(\beta/\alpha\right)^{2}\right\rceil }\cdot\gridstar$. Consider the iteration:
\[
\x_{t+1}=\qs_{\grid}\left(\x_{t}-\frac{\eta}{\beta}\grad{\x_t}\right)\,.
\]
In 
$T=\frac{10}{\eta}\cdot\frac{\beta}{\alpha}\ln\frac{f\left(\x_{0}\right)-\mathbb{E}f(\xstar_{r,\gridstar})}{\err}
$
iterations we obtain a point $\x_{T}$ satisfying $\mathbb{E}f\left(\x_{T}\right)-\mathbb{E}f(\xstar_{r,\gridstar})\leq\err$.
\end{restatable}

\paragraph{Discussion.} 
To understand the convergence of this method, let us establish as benchmark the expected value $\mathbb{E} f(\xstar_{r,\gridstar})$ of the best iterate on the lattice $\gridstar \mathbb{Z}^n + r\one$, where the expectation is taken over the shift $r$. Our method finds a point in a slightly finer grid $\x_T \in \grid \mathbb{Z}^n + r'\one$, such that in expectation over the randomness in the algorithm, the value of the function is at most $\epsilon$ larger than our benchmark. The sacrifice we have to make in exchange for this surprisingly strong convergence is an increase in resolution for the iterates we maintain, which is dependent, among others, on the condition number of $f$.

Since our method works with stochastic gradients, we can additionally quantize gradients to further reduce communication. In fact any quantization method that compresses gradients to an unbiased estimator with low variance can be directly plugged into Theorem~\ref{thm:mainthm}.

We state in the following corollary a generic bound for quantized gradients, which highlights the trade-off between variance and communication for the quantization method.

\begin{cor}[Gradient Quantization]
\label{cor:qwqg}
Let $\alpha, \beta, \gridstar, \err, b > 0$ and $\sigma, \sigmagrad \geq 0$ be real parameters. Let $f:\mathbb{R}^{n}\rightarrow\mathbb{R}$ be a
$\beta$-smooth and $\alpha$-PL function, with access to a stochastic
gradient estimator $\grad{\x}$, i.e. $\mathbb{E}\left[\grad{\x}\vert\x\right]=\nabla f\left(\x\right)$
with bounded variance $\mathbb{E}\left\Vert \grad{\x}-\nabla f\left(\x\right)\right\Vert _{2}^{2}\leq\sigma^{2}$.
Let $\qf : \mathbb{R}^n \rightarrow \mathbb{R}$ be a gradient quantizer which for any stochastic gradient $\grad{\x}$ encountered during the execution of the algorithm, ensures:
\begin{enumerate}
\item unbiased estimator: $\mathbb{E}\left[ \qf(\grad{\x}) \vert \grad{\x}\right] = \grad{\x}$,
\item variance: $\mathbb{E}\left[ \left\Vert\qf(\grad{\x})  - \grad{\x} \right\Vert_2^2 | \grad{\x} \right] \leq \sigmagrad^2$, 
\item requires $b$ bits to communicate $\qf(\grad{\x})$.
\end{enumerate}
 
For each $r\in\intgs$, let $\xstar_{r,\gridstar}$
be any minimizer of $f$ over $\gridstar\mathbb{Z}^{n}+r\cdot\one$.
Let $\eta=\min\left\{ \frac{3}{10}\frac{\err\alpha}{\sigma^{2}+\sigmagrad^2},1\right\}$, $\grid=\frac{\eta}{\left\lceil 16\left(\beta/\alpha\right)^{2}\right\rceil }\cdot\gridstar$, and consider the iteration:
\[
\x_{t+1}=\qs_{\grid}\left(\x_{t}-\frac{\eta}{\beta}\qf\left(\grad{\x_t}\right)\right)\,.
\]
In $T=\frac{10}{\eta}\cdot\frac{\beta}{\alpha}\ln\frac{f\left(\x_{0}\right)-\mathbb{E}f\left(\xstar_{r,\gridstar}\right)}{\err}$
iterations we obtain a point $\x_{T}$ satisfying $\mathbb{E}f\left(\x_{T}\right)-\mathbb{E}f(\xstar_{r,\gridstar})\leq\err$.
Furthermore, the entire algorithm requires $O\left(b \cdot \frac{\sigma^{2}+\sigmagrad^2}{\err\alpha}\frac{\beta}{\alpha}\ln\frac{f\left(\x_{0}\right)-\mathbb{E}f(\xstar_{r,\gridstar})}{\err}\right)$
bits to communicate the quantized gradients.
\end{cor}

We notice that since $b$ and $\sigmagrad$ are inversely associated, we can establish a trade-off between the number of iterations and the total communication. As a matter of fact, this trade-off kicks in only at the point where the variance of the quantized gradient  estimator becomes as large as that of the stochastic gradient, as the number of iterations scales linearly with $\sigma^2 + \sigmagrad^2$.
For example, given a resolution parameter $\gridgrad > 0$, a simple gradient quantization scheme such as the one employed by QSGD, quantizes gradient entries to $\gridgrad \mathbb{Z}^n$, while guaranteeing $\sigmagrad^2 \leq \gridgrad \gradnorm$, where $\gradnorm \geq \|\grad{\x}\|_1$, and $b = O({\gradnorm}/{\gridgrad}\cdot\left(\ln n+\ln\gradnorm\right))$ bits required for communication. See a more detailed discussion in Section~\ref{sec:quantgrad}. By varying $\gridgrad$ we distinguish between the extreme cases, corresponding to the scenarios where we the quantized gradients are dense and sparse, respectively. While the total communication does not improve by varying $\gridgrad$, by doing so we are able to reduce the communication performed in each iteration, in practice~\cite{alistarh2017qsgd}.\\
\paragraph{Dense gradients:} setting $\gridgrad={\sigma^{2}}/{\gradnorm}$, we obtain exactly
the same number of iterations as in the basic case without quantized
gradients, but the communication per iteration is reduced to $O({\gradnorm^{2}}/{\sigma^{2}}\cdot\left(\ln n+\ln{\gradnorm}\right))$. \\
\paragraph{Sparse gradients:}
setting $\gridgrad=\gradnorm$, the number of iterations scales
with $\max\left\{ \sigma^{2},\gradnorm^{2}\right\} $ rather than $\sigma^{2}$,
but the pre-step communication is reduced to $O\left(\ln n + \ln \gradnorm\right)$
bits.

\subsection{Analysis Overview}
\label{sec:analysis}

Let us briefly explain the intuition behind our theoretical analyses. We view our iteration as a version of projected gradient descent, where iterates are projected onto the non-convex domain of quantized vectors.
In general, when the domain is convex, projections do not hurt convergence. But in our setting the distance to the optimal solution can increase and drastically affect the loss. However, we can show a trade-off between how much this distance increases and the ratio between the target and optimal resolution $\grid/\gridstar$.

To understand this better, consider a point $\xp$ obtained by taking a step $\xp = \qs_{\grid}(\x - \frac{1}{\beta} \nabla f(\x) )$. Using smoothness, we can verify that this significantly decreases the loss, provided that the quantization operator does not perturb its input by too much in $\ell_2$ norm. Formally, using Lemma~\ref{lem:progress-1-step} we see that
\begin{align}\label{eq:progperstep}
&f(\xp) \leq f(\x) - \frac{1}{2\beta} \|\nabla f(\x)\|_2^2
\\
&+ \frac{\beta}{2}\left\Vert \qs_{\grid}\left(\x - \frac{1}{\beta} \nabla f(\x)\right) - \left(\x - \frac{1}{\beta } \nabla f(\x)\right) \right\Vert_2^2\,.\nonumber
\end{align}
Since compared to a vanilla gradient method, this suffers a reduction in the progress made in a single iteration, we can force this to be significantly smaller, so as not to undo more than a fraction of the progress we would ideally make. To do so we notice that we can charge the last term in (\ref{eq:progperstep}) to the current error in function value, and we can make this dependence be arbitrarily small by using a finer resolution $\grid$ for our quantization grid.
This is captured by the following crucial lemma, which we prove in Section~\ref{sec:quant-err-opt-rel-shift-pf}.

\begin{restatable}{lem}{progressonestep}
\label{lem:quant-err-opt-rel-shift}Let $\gridstar>\grid>0$, such
that $\gridstar/\grid\in\mathbb{Z}$. Let $\x\in\mathbb{R}^{n}$,  and for all $r\in\intgs$, let an arbitrary $\xstar_{r,\gridstar}\in\gridstar\mathbb{Z}^n+r\one$.
Then 
\[
\mathbb{E}\left[\left\Vert \qs_{\grid}\left(\x\right)-\x\right\Vert _{2}^{2}\right]\leq\frac{\grid}{\gridstar}\mathbb{E}_{r}\left[\left\Vert \xstar_{r,\gridstar}-\x\right\Vert _{2}^{2}\right]\,.
\]
\end{restatable}
Using Lemma~\ref{lem:quant-err-opt-rel-shift}, together with the $\alpha$-PL condition, we can charge the extra error term to 
\[
\frac{\beta}{2} \cdot \frac{\grid}{\gridstar}  \mathbb{E}_r \left[  \frac{2}{\alpha}  \left( f\left(\x - \frac{1}{\beta} \nabla f(\x) \right)- f(\xstar_{r,\gridstar})\right)\right]\,,
\]
where $\xstar_{r,\gridstar}$ are picked to be the best minimizers in $\gridstar \mathbb{Z}^n + r\one$. 
To simplify the exposition and highlight the main ideas, let us assume that $\mathbb{E}_r[f (\xstar_{r,\gridstar})] = f(\xstar)$.
Since by the $\alpha$-PL condition we know that the gradient norm is large compared to the error in function value, we conclude that 
\begin{align*}
f(\xp) &- f(\xstar) \leq f(\x) - f(\xstar) - \frac{\alpha}{\beta} \left(  f(\x) - f(\xstar) \right) \\
&+ \frac{\beta}{\alpha} \cdot \frac{\grid}{\gridstar} \left( f\left(\x - \frac{1}{\beta} \nabla f(\x)\right) - f(\xstar) \right) \\
&\leq \left(f(\x) - f(\xstar)\right) \left(1 - \frac{\alpha}{\beta} + \frac{\beta}{\alpha} \frac{\grid}{\gridstar}\right)\,.
\end{align*}
This shows that by setting the $\grid \leq \gridstar \cdot (\alpha/\beta)^2/2$, in each iteration the error contracts by a $1-\Theta(\alpha/\beta)$ factor, which allows us to conclude that this algorithm converges linearly to a minimizer. We provide full proofs in Section~\ref{sec:proofs}.

\section{QSDP Implementation}
\label{sec:implementation}

\subsection{Overview}

We implemented a practical version of the QSDP algorithm described in the previous section, supporting both weight and gradient quantization, 
in Pytorch~\cite{pytorch} starting from the PyTorch FSDP support. 
Our implementation uses the CGX framework~\cite{CGX2022} as a communication backend, 
to which we added support for quantized AllGather and Reduce-Scatter collectives. 

In the original FSDP implementation, layers are packed into groups: weights and gradients of layers in the same group are concatenated before  communication. 
In QSDP, we compress layers separately, filtering out normalization layers and biases, which are communicated in full precision. 
This filtering is implemented at the level of the CGX communication backend. The quantized AllGather and Reduce-Scatter operations are implemented by leveraging peer-to-peer NVIDIA NCCL primitives. 
For multi-node (inter-server) communication, we used hierarchical versions of the algorithms, to reduce the size of inter-node transmissions.

One important optimization regards the granularity at which quantization is performed. 
Specifically, applying quantization over large tensors suffers from scaling issues, which results in accuracy degradation. 
To address this, we perform compression independently into equally-sized ``buckets'' of fixed size, 
and compress each bucket independently. 
This approach sacrifices compression by a negligible amount (as we need to transmit min-max scaling meta-information for each bucket), 
but helps avoid loss in terms of model quality. 

Bucketing (or grouping) on the weights is known to be necessary for good accuracy when quantizing \emph{pre-trained} LLMs~\cite{dettmers2022case}. 
It is also justified theoretically (Theorem~\ref{thm:mainthm}), as it both reduces compression variance, and allows us to explore solutions over finer-grained lattices. 
We observed experimentally that bucket size 1024 provides a good balance between compression and accuracy, and use it as a universal hyper-parameter.
In the context of this optimization, we observed that the impact of stochasticity in the quantization becomes minimal.

\subsection{Learned Weight Quantization}
\label{sec:learned_levels}

We now describe an additional (optional) optimization, which allows us to further reduce practical bit-width, at little to no accuracy loss. 
The motivating observation behind this optimization is that the quantization schemes we use for weights and gradients assume \emph{uniform} locations of the quantization levels. 
Yet, this uniform grid does not take the distribution of values into account. 
The idea of adapting the locations of quantization levels to the data distribution has already been studied~\cite{zhang2017zipml, faghri2020adaptive}.
However, existing dynamic-programming approaches~\citet{zhang2017zipml} have high computational cost (quadratic in the number of data points); thus, we use a fast version of gradient-descent-based optimization over the quantization levels~\cite{faghri2020adaptive}. 

 The goal of the distribution-aware quantizer in  Algorithm~\ref{algo:sgd_levels} is to select new locations for a fixed number of quantization points,  and weight values, so as to minimize the error introduced by quantization. 
 The algorithm runs iteratively across all values, finds the locations of quantization points for each value, and updates the \emph{quantization points} using the gradient descent update rule.
We run this heuristic periodically after a warmup period, separately for the weights and gradients of each layer. 
We save the derived locations of the quantization levels, and use them for quantization until the next re-computation. 

\begin{algorithm}

{\small
\caption{Gradient-based Optimization of the Levels}
\label{algo:sgd_levels}
\begin{algorithmic}[1]
\STATE {\bfseries Input:} values $V$, initial levels $Q_0$, learning rate $\alpha$.
\STATE {\bfseries Output:} optimized quantization levels $Q$.
\STATE Normalize values $V$ bucket-wise.
\FOR {each value $v_i$ from $V$}
    \STATE  $q_i = \mathsf{find\_closest}(v_, Q_i)$ \hfill // Quantize using current level
    \STATE $q_i = q_i - \alpha(q_i - v_i)$ \hfill // Update chosen quantization level
\ENDFOR
\end{algorithmic}
}
\end{algorithm}

\section{Experimental Validation}
\subsection{Experimental setup}
\paragraph{Infrastructure.} We evaluate QSDP for training GPT-scale LLMs using multiple cloud-grade Amazon EC2 \texttt{p3dn.24xlarge} machines, with 8 V100 SXM2 GPUs each. Each GPU has 32GB memory. The inter-GPU interconnect is provisioned by NVLinks with 200Gbps, while the inter-server bandwidth is of 100 Gbps. 

\paragraph{Environment and Tasks.}
We use the official NGC PyTorch 22.05-py3 Docker image with PyTorch 1.12, CUDA 11.6.2, NCCL 2.12, and the MosaicML Composer library (version 0.12), as well as a fork of the CGX communication library~\cite{CGX2022}.
All experiments were run with with MosaicML Large Language Models implementation~\cite{MosaicMLBench}. The benchmarks run the pre-training of different version of GPT-family models~\cite{radford2018improving, brown2020language}, varying the sizes of the models, on the C4 dataset~\cite{2019c4dataset}. 
Specifically, we examine accuracy on GPT models with 125M, 350M, and 1.3B parameters. 
For  benchmarks, we use 4 servers with 8 GPUs each. See Appendix~\ref{app:training_details} for training details.

\paragraph{Baselines.}
As a baseline we use training with default parameters, which is already highly-optimized by MosaicML~\cite{MosaicMLBench}. 
We note that directly using INT8 quantization, without bucketing, resulted in very poor accuracy, and therefore we do not use it as a baseline. 
In terms of communication, the baseline transmits weights in full (FP32) precision, and gradients in half (FP16) precision. In QSDP experiments, we \emph{do not} modify any hyperparameters. We convert gradients to full precision before quantization. 
For all timing experiments, the reported numbers are averaged over 50 training steps after warm-up of 10 iterations.
Our main accuracy measure is \emph{perplexity}, which is known to be a very stringent accuracy measure in this setting, and correlates extremely well with zero-shot performance~\cite{dettmers2022llm}.

\paragraph{Accuracy Recovery.} We first examine the effect of quantization on model quality, i.e. final model perplexity, in the end-to-end experiments. The default bit-width for weights and gradients quantization is 8 bits, using 1024 bucket size, which we illustrate as W8G8. 
We communicate normalization layers and biases in full precision.
We emphasize that straightforward round-to-nearest or stochastic quantization \emph{does not converge} to reasonable final perplexity in this setup: 
Naive quantization without bucketing loses more than 2 units of perplexity on GPT-125M, a model on which W8G8 with 1024 bucket size \emph{improves} perplexity. 

The accuracy results are presented in Table~\ref{tab:model_ppls}. 
The QSDP final perplexity is almost identical to that of regular training, and  QSDP can even \emph{slightly improve} the baseline accuracy. We stress that we did not perform any parameter tuning: quantization parameters are the same across all layers.

\paragraph{End-to-end Speedup.}
For end-to-end training speedup improvements, we use multi-node GPT pretraining under standard hyperparameters. 
We examine speedup for different inter-node bandwidths: 10 Gbits, 50 Gbits and 100 Gbits. 
For that, we artificially reduce input-output bandwidth on each node, using the UNIX \texttt{tc} tool~\cite{tcunix}. The results are presented in Figure~\ref{fig:training_times}. 
First, please notice that standard FSDP training has a non-trivial bandwidth bottleneck even at 100Gbps bandwidth as we increase model size, and that this bandwidth bottleneck can dominate training time on the lower 10Gbps bandwidth. 
Second, the running time of QSDP is \emph{essentially constant} across all three scenarios, showing that it has essentially removed the bandwidth bottleneck. 
More precisely, QSDP outperforms the baseline by up to 15\% in the 100Gbps scenario (a non-trivial reduction of 12 hours of training time or 1.5k\$ of cloud costs\footnote{price of 12 hours training on 4 AWS p3dn.24xlarge instances}), and by 2.25x in the 10Gbps scenario.

\begin{figure}[t]
 \centering
 \includegraphics[width=0.43\textwidth,trim={0.2cm 0.29cm 0 0},clip]{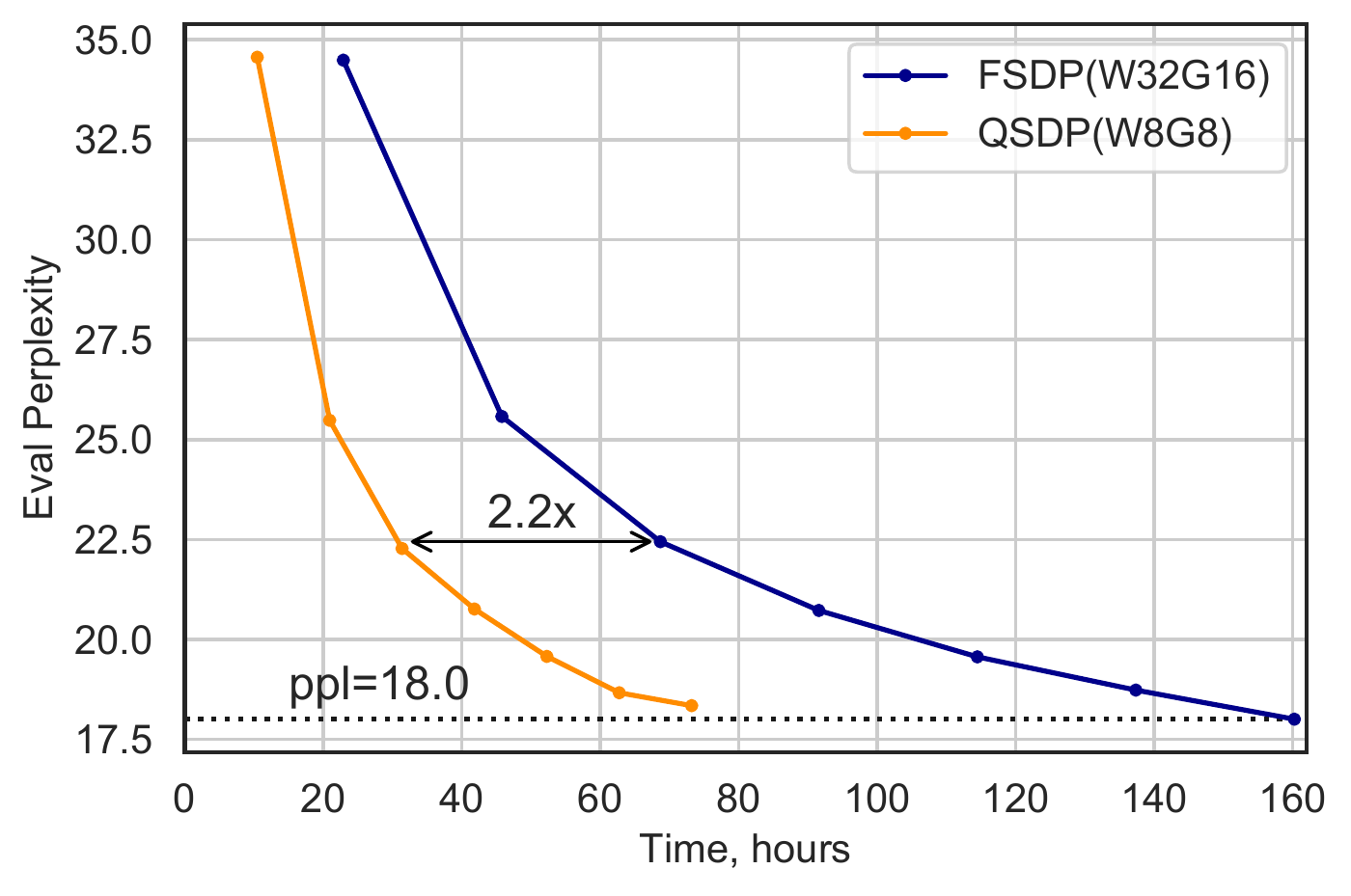}
 \caption{Perplexity vs time for standard FSDP (FP32 weights and FP16 gradients) and QSDP (both weights and gradients quantized to 8 bits) for the 1.3B model in the 10Gbps bandwidth setup.}
  \label{fig:ppl_vs_time}
\end{figure}

\paragraph{Learned quantization.}
We examined the performance of learned quantization for the small 125M parameters model. 
We ran the optimization algorithm after 400, 1900 and 3800 training steps, and noticed that optimizing the locations of quantization levels has no effect for bit-widths higher than 6 bits, but leads to noticeable improvements for lower bit-widths. Please see Table~\ref{tab:learned_levels}. 
Learned weight quantization allows to improve the final model performance for different weight and gradient quantization parameter pairs, reaching perplexities that are close to the baseline. 
Specifically, using learned quantization results in reaching the highest compression ratio for weights and gradient in training (i.e. 5 and 4 bits respectively) without substantial accuracy impact. We expand upon these experiments in Appendix~\ref{app:learned_levels}.

\begin{table}[h]
\caption{Perplexities recoveries for different models end-to-end training using QSDP. Weights and gradients quantized to 8 bits, uniform quantization.}
\label{tab:model_ppls}
\vskip 0.15in
\begin{center}
\begin{small}
\begin{tabular}{|l|c|c|c|}
\hline
 & 125M & 350M & 1.3B \\ \hline
Baseline & 35.81 & 23.94 & 18.00 \\ \hline
QSDP & 35.58 & 23.95 & 18.34 \\ \hline
\end{tabular}
\end{small}
\end{center}
\vskip -0.1in
\end{table}

\begin{table}[h]
    \caption{Final perplexities of training 125m GPT-2 model with combinations of weights and gradients low-bits uniform (not learned) quantization.}
    \label{tab:ppl_125m}
    \begin{center}
    \begin{small}
    \begin{tabular}{|c|c|c|c|c|c|}
    \hline
    \backslashbox{Weights bits}{Gradients bits} & 6 & 5 & 4 \\ \hline
    6 & 35.74 & 36.08 & 35.84 \\ \hline
    5 & 36.01 & 35.94 & 36.36  \\ \hline
    4 & 37.11 & 37.38 & 37.61 \\ \hline
    \end{tabular}
\end{small}
\end{center}
\end{table}

\begin{figure}[t]
 \centering
 \includegraphics[width=0.43\textwidth,trim={0.2cm 0.28cm 0 0.1cm},clip]{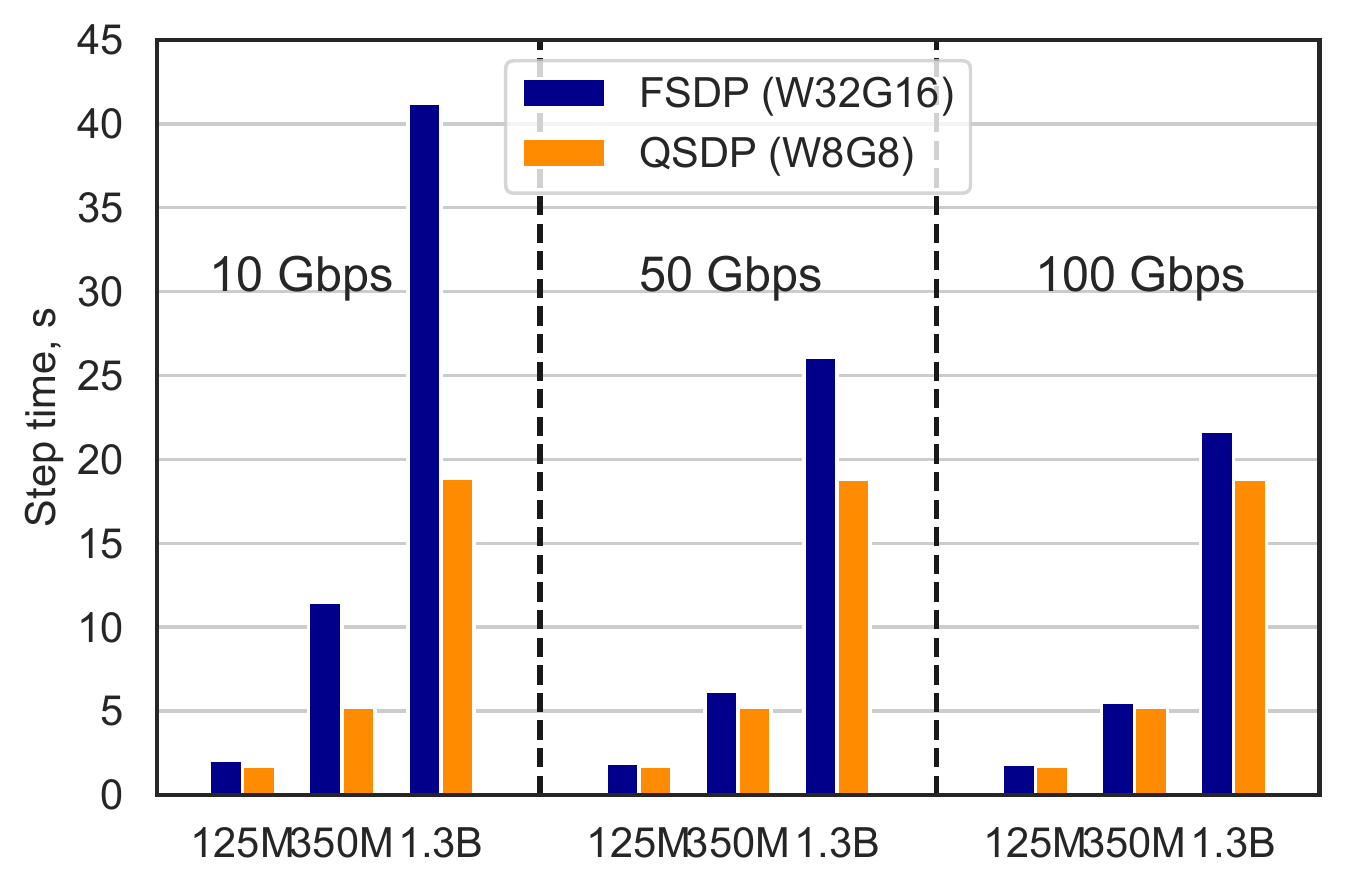}
 \caption{Training step time for different models at various inter-node bandwidth with and without QSDP enabled. The fact that QSDP step time is constant across considered bandwidths means that QSDP successfully tackles bandwidth bottlenecks.}
 \label{fig:training_times}
\end{figure}

\begin{table}[h]
\caption{ 
Final perplexities of low-bits quantization of 125m GPT-2 model using the learned quantization levels. Learned quantization in the W6G4 configuration provides lower perplexity than the baseline.
}
\label{tab:learned_levels}
    \begin{center}
    \begin{small}
    \begin{tabular}{|c|c|c|c|c|c|}
    \hline
            & baseline               & w6g4  & w5g4  & w4g4  & w4g32  \\ \hline
    Uniform & \multirow{2}{*}{35.81} & 35.81 & 36.34 & 37.61 & 37.11  \\ \cline{1-1} \cline{3-6}
    Learned &                        & \textbf{35.75} & \textbf{36.01} & \textbf{36.94} & \textbf{36.55} \\ \hline
    \end{tabular}
\end{small}
\end{center}
\end{table}

\section{Conclusion}

Motivated by the efficient distributed training of large language models, we have explored the feasibility of fully-quantized training for such models, with respect to both weights and gradients. 
This led to an interesting new analysis, showing that SGD can indeed converge with strong convergence guarantees even with quantized iterates, as long as a good quantized solution exists. 

Complementing this analysis, we proposed QSDP, a quantized an extension of the popular Fully Sharded Data Parallel (FSDP) distributed training approach, in which \emph{all transmitted state is in quantized form}. 
We also provided a highly-efficient implementation of QSDP in Pytorch, which we showed to  successfully eliminate the bandwidth bottleneck in large-scale distributed training of modern language models, without sacrificing accuracy. 
Specifically, our experimental validation across three model sizes showed that training with QSDP reaches up to 2.2$\times$ speedup.

Our results suggest that communication compression can  be an effective tool in the context of novel distribution schemes motivated by large-scale training. Specifically, 
we believe we are the first to show both convergence guarantees and strong practical performance for simple \emph{weight compression} schemes being applied during SGD-based training, which should motivate further work in this direction. In particular, an interesting extension of our work would be to examine whether the lower-precision weight representation can also be exploited for faster runtimes.

\section*{Acknowledgements}
AV acknowledges the support of the French Agence Nationale de la Recherche (ANR), under grant ANR-21-CE48-0016 (project COMCOPT), the support of Fondation Hadamard with a PRMO grant, and the support of CNRS with a  CoopIntEER IEA grant (project ALFRED).

\bibliographystyle{plainnat}
\bibliography{references}

\newpage
\appendix

\section{Training details.}
\label{app:training_details}
For training of GPT-2 models we were using MosaicML~\cite{MosaicMLBench} examples. The global batch size for 125M and 350M models was 256, for 1.3B - 512, resulting in 4 gradient accumulations at each iteration. 
For all models AdamW optimizer was used, the optimizer parameters are presented in the Table~\ref{tab:opt_params}. 125M model was trained in 4800 steps, 350M model in 13400 steps, 1.3B model in 14000 steps.

\begin{algorithm}[t]
{\small
\caption{Pseudocode of QSDP for a Fixed Layer}
\label{algo:qsdp}
\begin{algorithmic}[1]
\STATE {\bfseries Input:} worker $p$, layer input $x_p$, worker weight partition $w_p$.

\FUNCTION{\textbf{ExecuteForwardPass}}
    \STATE \textcolor{blue}{$qw_p \gets \mathsf{QuantizeWeights}( w_p )$} \hfill// Quantize $p$'s weights 
    \STATE $qw \gets  \textcolor{orange}{\mathsf{AllGather}( qw_{i} \textnormal{ for all $i$ } )}$ \hfill // Collect quantized weights 
    \STATE $o_p \gets \mathsf{Layer}( qw, x_p )$ \hfill // Compute output for $p$
    \STATE \textcolor{gray}{$\mathsf{free}(qw)$} \quad \hfill // Discard aggregated layer weights
\ENDFUNCTION

\FUNCTION{\textbf{ExecuteBackwardPass}}
    \STATE \textcolor{blue}{$qw_p \gets \mathsf{QuantizeWeights}( w_p )$}  \hfill // Quantize $p$'s weights 
    \STATE $qw \gets \textcolor{orange}{\mathsf{AllGather}( qw_{i} \textnormal{ for all $i$ } )}$ \hfill // Collect quantized weights 
    \STATE $g_p \gets \mathsf{Gradient}( qw, o_p )$ \hfill // Compute gradient for $p$
    \STATE \textcolor{gray}{$\mathsf{free}(qw)$} \quad \hfill // Discard aggregated layer weights
    \STATE \textcolor{blue}{$qg_p \gets \mathsf{QuantizeGradients}( g_p )$}  \hfill // Quantize $p$'s gradient
    \STATE $qg_p \gets \textcolor{orange}{\mathsf{ReduceScatter}( qg_{i} \textnormal{ for each $i$} )}$ \hfill // Distribute gradients
    \STATE $w_p \gets  \mathsf{WeightUpdate}(qg_p, w_p) $ \hfill // Update $p$'s weights
    \STATE \textcolor{gray}{$\mathsf{free}(qg)$} \quad \hfill // Discard aggregated gradients
\ENDFUNCTION

\end{algorithmic}
}
\end{algorithm}

\begin{table}[h]
\caption{ AdamW optimizer parameters.}
\label{tab:opt_params}
    \begin{center}
    \begin{small}
    \begin{tabular}{|c|c|c|c|}
    \hline
            &125M  & 350M  & 1.3B  \\ \hline
    learning rate & 6e-4 & 3e-4 & 2e-4 \\ \hline
    betas & 0.9, 0.95 & 0.9, 0.95 & 0.9, 0.95 \\ \hline
    epsilon & 1e-8 & 1e-8 & 1e-8 \\ \hline
    \end{tabular}
\end{small}
\end{center}
\end{table}

\section{Network overhead experiments}
In order to evaluate the effect on communications in FSDP training we conducted the synthetic experiment which reduces the bandwidth costs in each iteration. Specifically, given the buffer of size $N$ which is about to be communicated, and compression ratio $\gamma$ we only transmit the first $N/\gamma$ elements. The results for our setup (4 x 8V100-32G GPUs) at different internode bandwidths is shown in the Figures.\ref{fig:fake_compression}, communication weights and gradients are reduced to the same compression ratio. We see that the most effect of compression is reached as expected for the largest 1.3B model and at lowest bandwidth. However, one can get around 80\% speedup at high bandwidth when up to 8x compression ratio is applied. Also, we notice that 8$\times$ compression almost reaches the ideal scaling for large model and has a evident overhead over the no-communication training in case of the small model. It infers that the large models have a bottleneck in bandwidth component of the communication and the small model has a dominating latency part.

To see the variance of the compression effects on weights and gradients we conducted the similar experiment for different combinations of compression ratio pairs (see ~\ref{tab:fake_1b}). We observe that weight compression gives more performance profits than gradient compression. This can be naturally explained by the fact that weights are communicated more frequently than gradients in FSDP (in this particular experiment weights are communicated 5 times per one gradient exchange) and the amount of transmissions per communication is similar.

The difference between the synthetic experiment and QSDP performance numbers with the same compression ratios can be justified by the performance inefficiency of NCCL point-to-point communication primitives on which QSDP compressed communication is based on - the compression overhead in our experiments was verified to be negligible (less than 1\% per iteration).

\begin{figure*}[h]
    \centering
    \setkeys{Gin}{width=0.33\linewidth}    
    \subfloat[125M\label{fig:fake-125m}]{\includegraphics{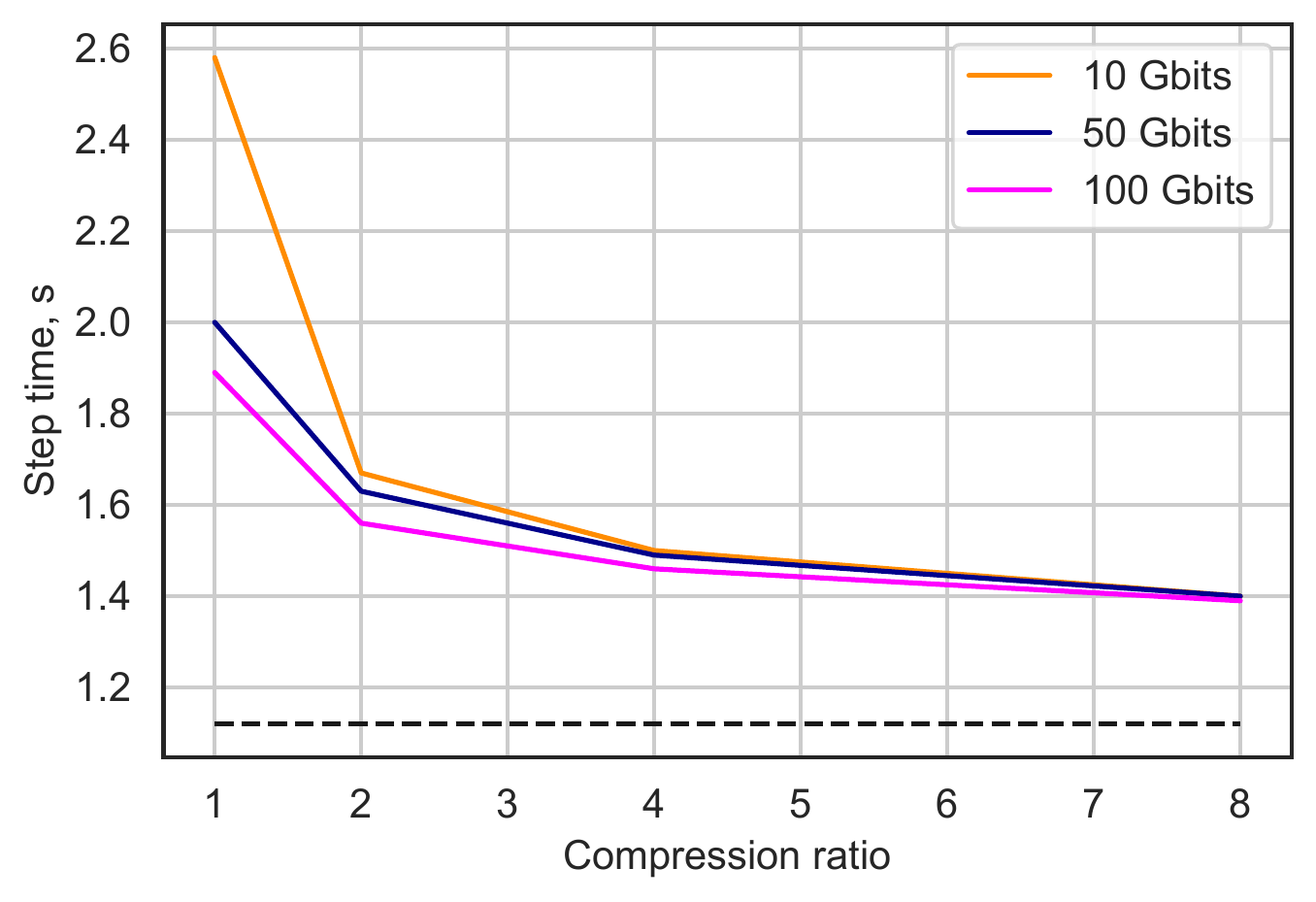}}\subfloat[350M\label{fig:fake-350m}]{\includegraphics{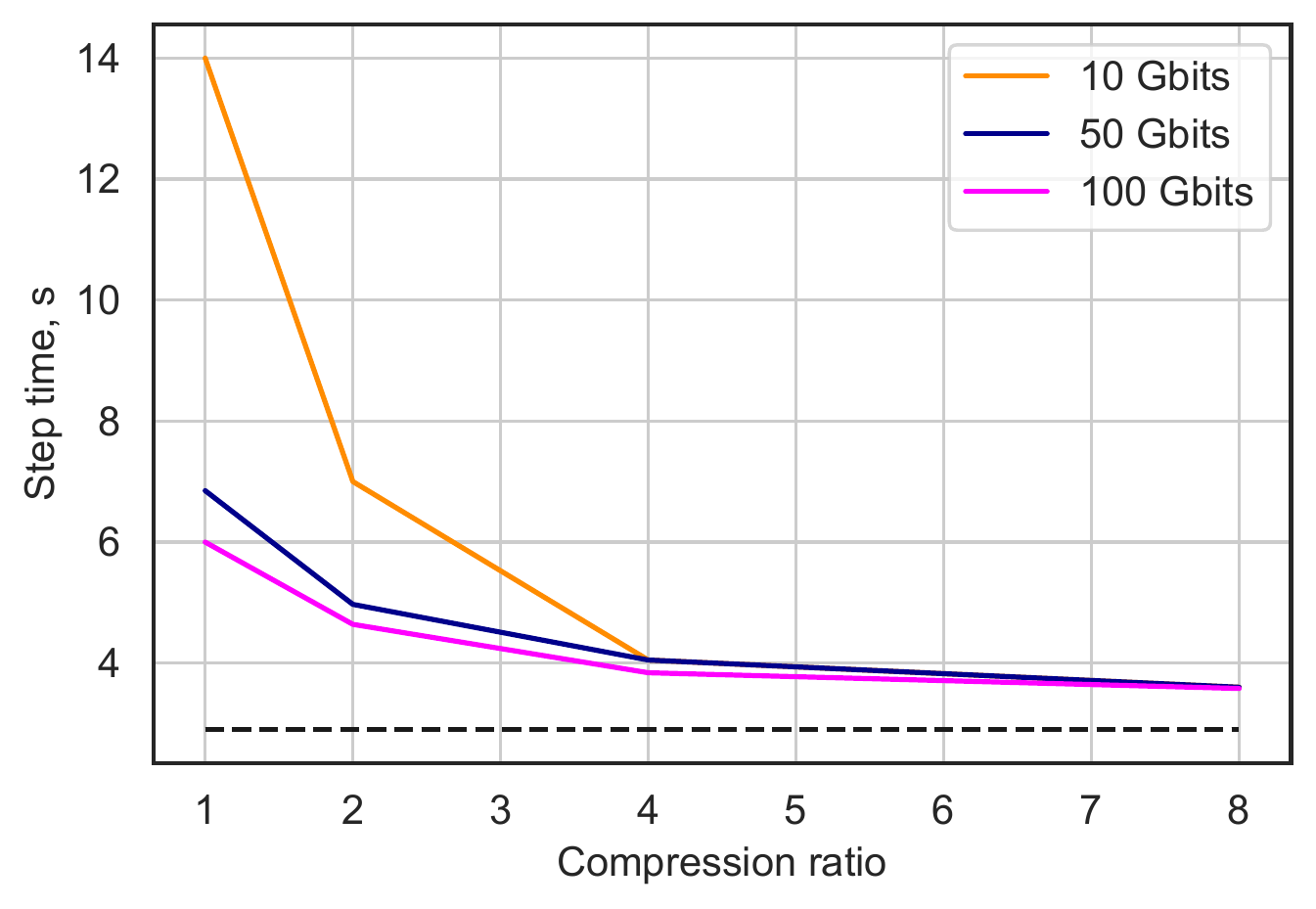}} \subfloat[1.3B \label{fig:fake-1b}]{\includegraphics{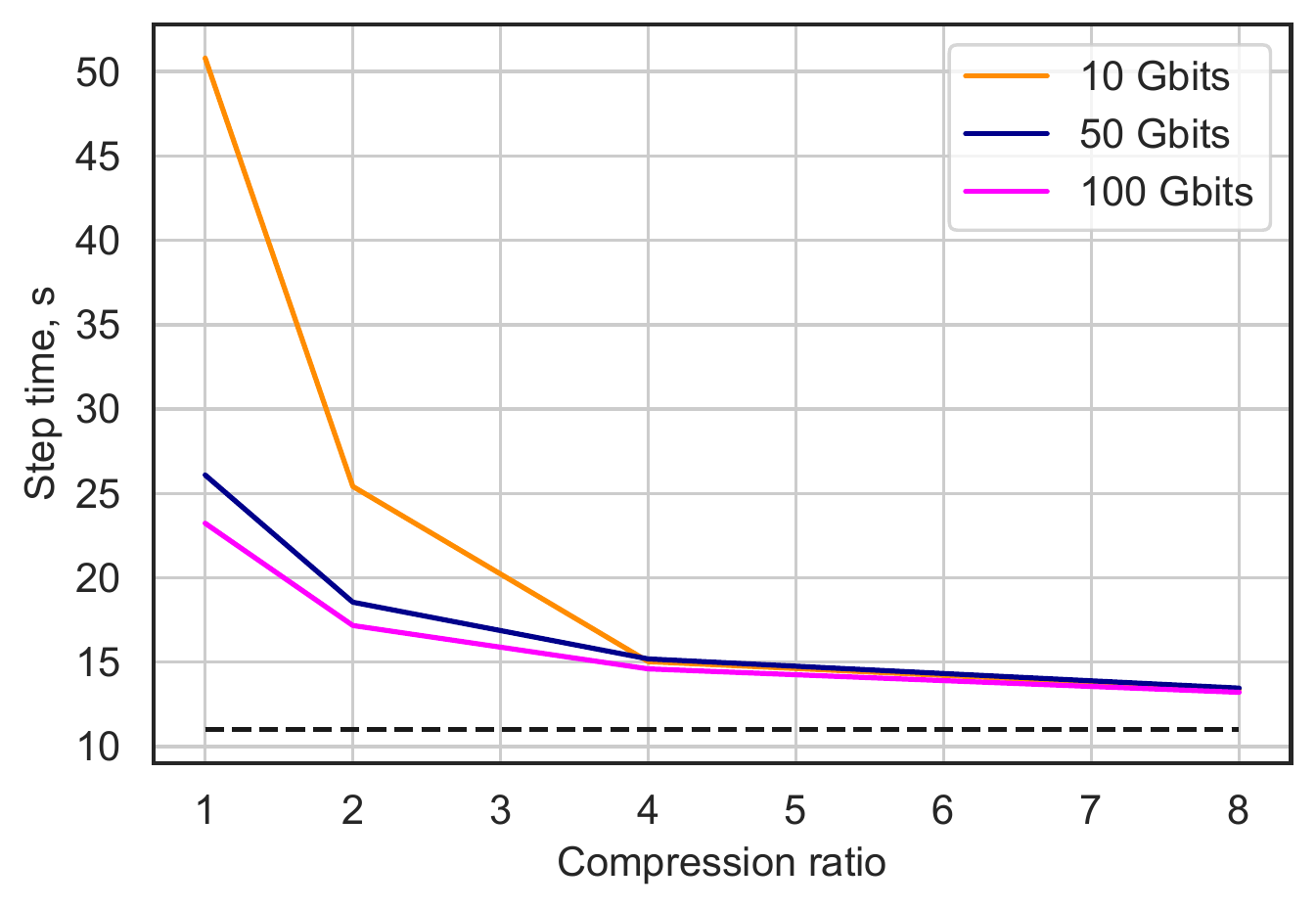}}
    \caption{{\small Compression vs average step time for different models at different inter-node bandwidths with fake compression (weights and gradients have the same compression ratio). Lower is better. The dashed line represents ideal scaling - training without communication.}}
    \label{fig:fake_compression}
 \end{figure*}

\begin{table}[h]
    \caption{Training step timings (in seconds) for 1.3B model at 100 Gbps bandwidth with various combinations of weights and gradient compression ratio.}
    \label{tab:fake_1b}
    \begin{center}
    \begin{small}
    \begin{tabular}{|c|c|c|c|c|c|c|}
    \hline
    \backslashbox{Weights ratio}{Gradients ratios} & 1 & 2 & 4 & 8 \\ \hline
    1 & 23.23 & 21.36 & 20.62 & 20.2   \\ \hline
    2 & 19.27 & 17.17 & 16.26 & 15.95  \\ \hline
    4 & 17.50 & 15.35 & 14.6  & 14.08  \\ \hline
    8 & 16.62 & 14.52 & 13.66 & 13.21  \\ \hline
    \end{tabular}
\end{small}
\end{center}
\end{table}

\section{Learned quantization}
\label{app:learned_levels}
We implemented stochastic gradient descent optimization of quantization levels in PyTorch. We use learning rate 0.01, batch size 1024. We run the learning for each layer larger than 1e5 parameters, for other layers uniform quantization was used.
We evaluate the quality of quantization levels by comparing L2 of compression error introduced by quantizing a buffer using the levels. We conducted the such evaluation for weights quantized to 5 bits and gradients quantized to 4 bits during the training of GPT 125M model. The results for one of the attention layers and LM head layer are shown in the Figures~\ref{fig:learned_compression_error_attn} and ~\ref{fig:learned_compression_error_lm}. The dashed vertical lines show the moment of running learning quantization levels algorithm. We see that compression error of learned quantization levels is constantly lower for the learned algorithm, and the lower bits-width (for gradients we use 4 bits quantizaton) the larger the gap between the considered methods. Also, we see that the compression error of the learned quantization only increases in sync with uniform quantization over time. It means that learning algorithm can be run only once, at the start of the training.

Also, we measured overhead of running learning algorithm for GPT 125M with weights quantized to 5 bits, gradients to 4 bits. The overhead of learning algorithm amounts to around 9 minutes, whereas the full training takes lasts 5 hours.

The extra experiments results with low bit-width quantization are shown in the Table.~\ref{tab:app_learned_levels}. The number doesn't show full perplexity recovery but they represent the improvements achieved by learned levels algorithm. We can see that with learned quantization levels one can reduce up to 3 units of perplexity.

\begin{figure*}[h]
    \centering
    \setkeys{Gin}{width=0.5\linewidth}    
    \subfloat[Weights]{\includegraphics{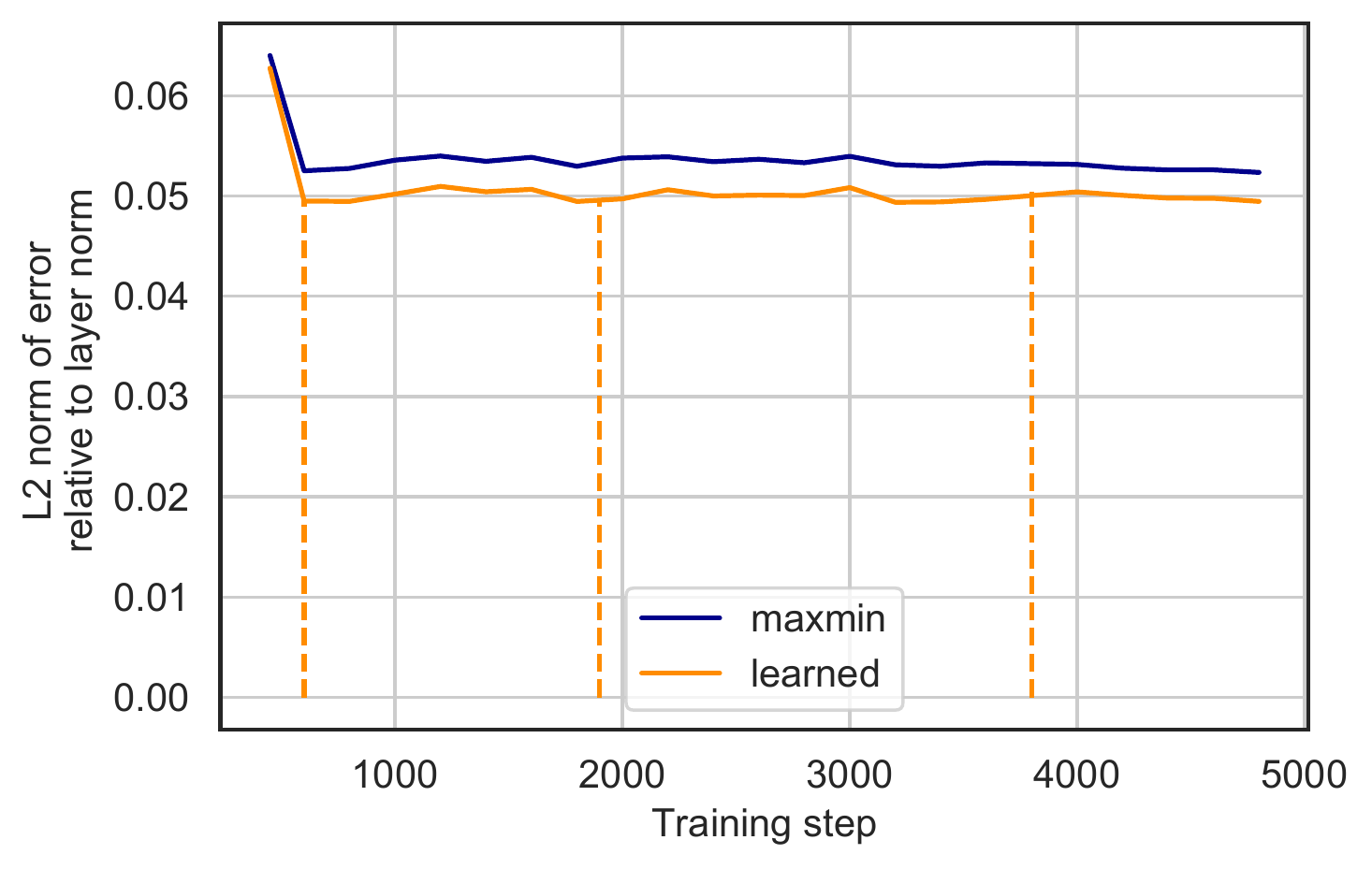}}\subfloat[Gradients]{\includegraphics{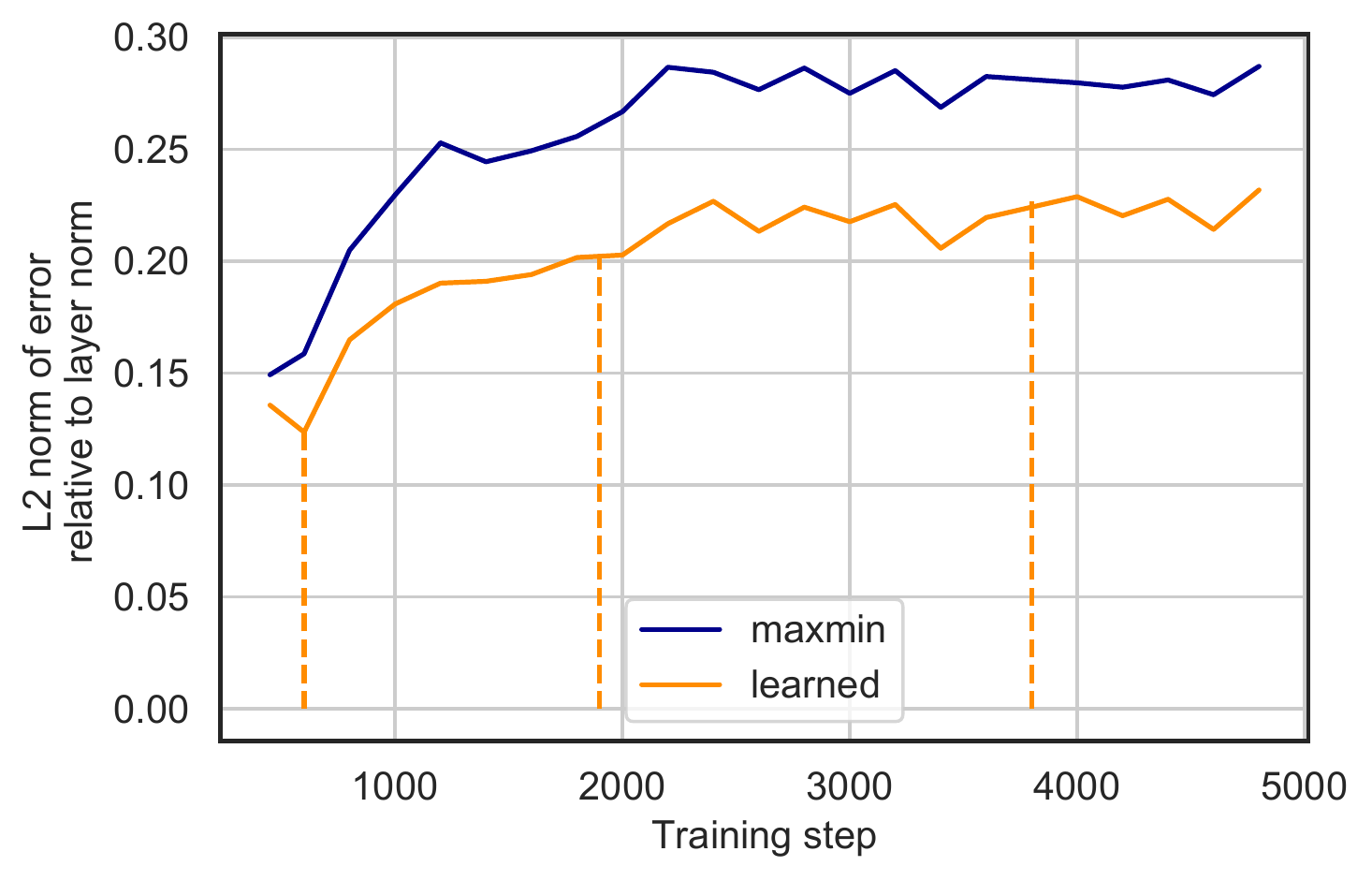}} \caption{{\small Compression error (L2 norm of the error relative to L2 norm of the input) comparison with learned quantization levels for attention layer of 125M model, W5G4 quantization.}}
    \label{fig:learned_compression_error_attn}
 \end{figure*}

\begin{figure*}[h]
    \centering
    \setkeys{Gin}{width=0.5\linewidth}    
    \subfloat[Weights]{\includegraphics{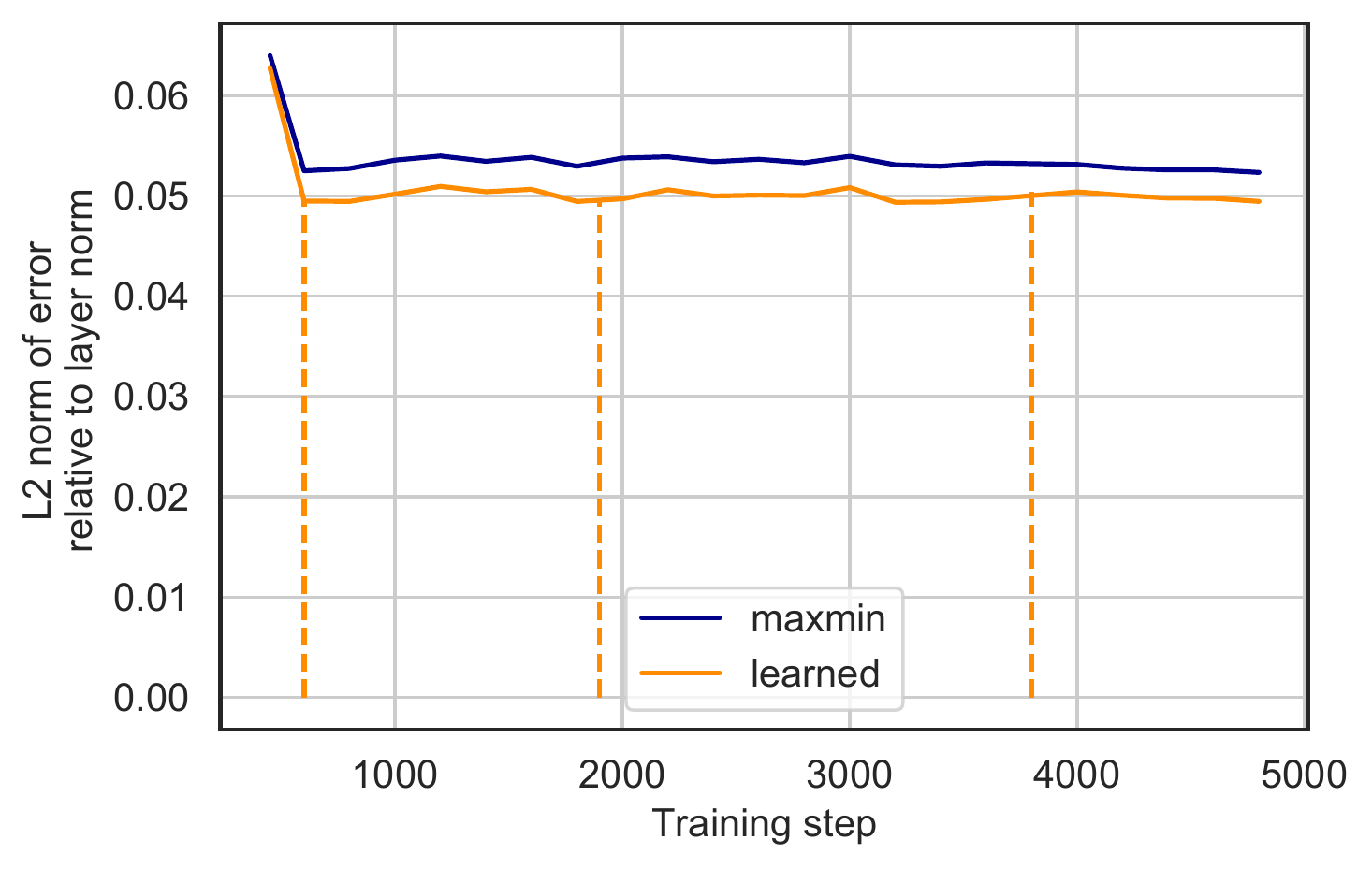}}\subfloat[Gradients]{\includegraphics{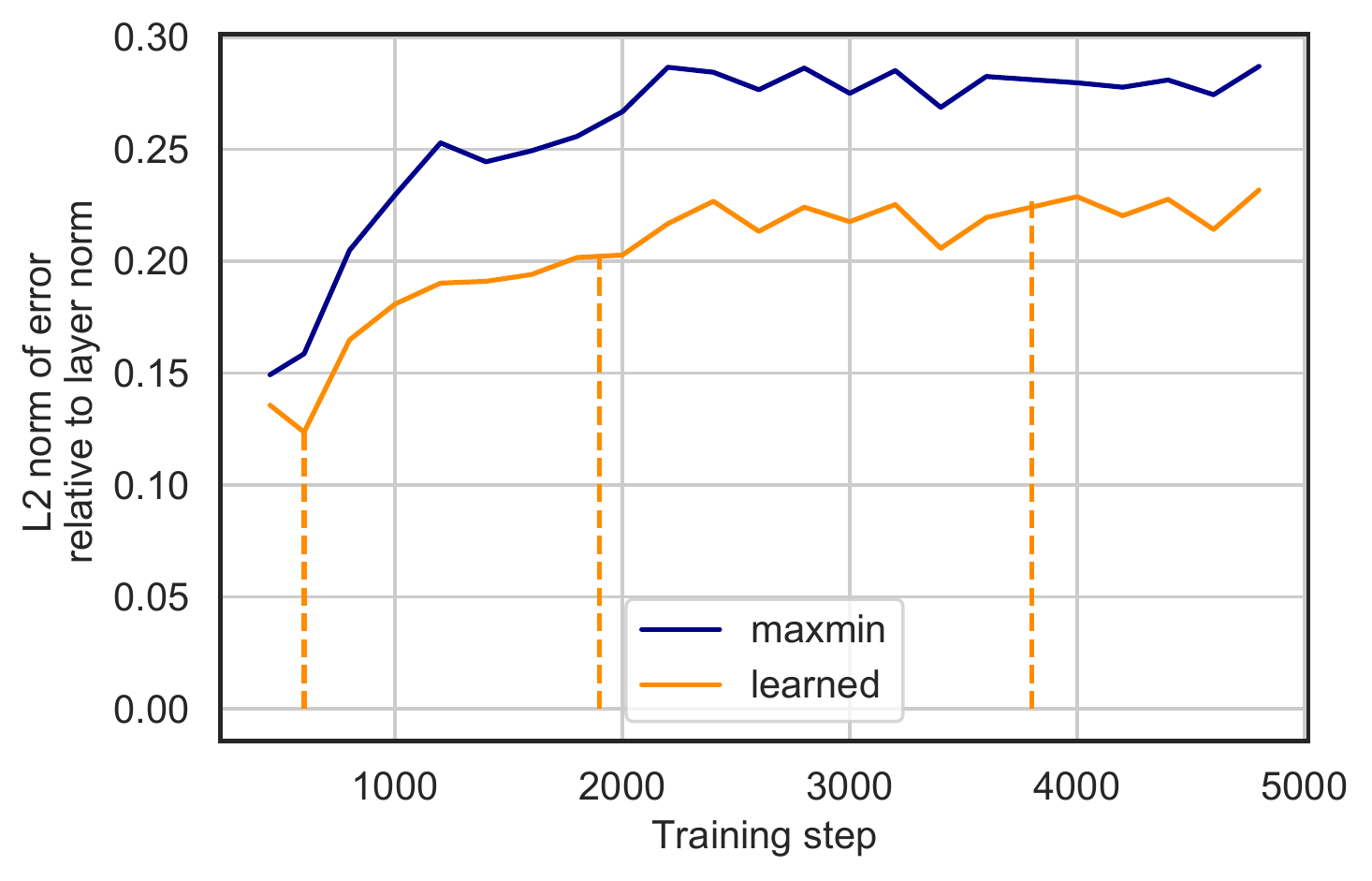}} \caption{{\small Compression error (L2 norm of the error relative to L2 norm of the input) comparison with learned quantization levels for LM-Head layerof 125M model, W5G4 quantization.}}
    \label{fig:learned_compression_error_lm}
 \end{figure*}

\begin{table}[h]
\caption{ 
Final perplexities of low-bits quantization of 125m GPT-2 model using the learned quantization levels.
}
\label{tab:app_learned_levels}
    \begin{center}
    \begin{small}
    \begin{tabular}{|c|c|c|c|c|c|}
    \hline
            & baseline               & w3g32  & w2g32  & w8g3  & w8g2  \\ \hline
    Uniform & \multirow{2}{*}{35.81} & 45.53 &  57.92 & 39.91  & 44.79  \\ \cline{1-1} \cline{3-6}
    Learned &                        & 42.31 &  56.54 & 37.72 & 44.65  \\ \hline
    \end{tabular}
\end{small}
\end{center}
\end{table}

\section{Convergence Proofs}
\label{sec:proofs}

In this section we provide the convergence analysis for our algorithms.

\subsection{Overview}

We use the notation and assumptions defined in Section \ref{sec:background}.
As all of our analyses revolve around bounding the progress made in
a single iteration, to simplify notation we will generally use $\x$
to denote the current iterate, and $\xp$ to denote the iterate obtained
after the generic update:
\[
\xp=\qs_{\grid}\left(\x-\frac{\eta}{\beta} \grad{\x}\right)\,,
\]
where $\beta$ is the smoothness parameter of $f$. In Section~\ref{sec:SGD-weight-quant} we will first prove convergence for the deterministic method, where we have direct access to the gradients of $f$. The analysis precisely follows the steps we described in Section~\ref{sec:background}.
Then, we extend this analysis to the case of stochastic gradients, and provide the full proof for Theorem~\ref{thm:mainthm}.
Finally, in Section~\ref{sec:quantgrad} we show that given an appropriate gradient quantization method with bounded variance, we can use it on top of our iteration to further reduce the required amount of communication, and thus prove Corollary~\ref{cor:qwqg}.

Before proceeding, we first formally analyze the quantization method we defined in Section~\ref{sec:background}, and show some additional properties that will be useful later.

\begin{lem}
\label{lem:mean-var-T}Let $\v\in\mathbb{R}^n$, and let $\grid>0$.
Then,
\begin{align*}
\mathbb{E}\left[\qs_{\grid}\left(\v\right)\right] & =\v\,,\\
\mathbb{E}\left[\left\Vert \qs_{\grid}\left(\v\right)-\v\right\Vert_2^{2}\right] & =\grid^{2}\cdot \sum_{i=1}^n \left\{ \frac{v_i}{\grid}\right\} \left(1-\left\{ \frac{v_i}{\grid}\right\} \right)\,,\\
\mathbb{E}\left[\left\Vert \qs_{r,\grid}\left(\v\right)-r\one\right\Vert _{0}\right] & \leq\left\Vert \v\right\Vert _{1}/\grid\,.
\end{align*}
\end{lem}
Since the proofs are technical, we defer them to Section~\ref{sec:mean-var-T-pf}. The most important feature of this quantization scheme is captured by Lemma~\ref{lem:quant-err-opt-rel-shift}, which is crucial for our convergence proof. We first restate it, and prove it formally in Section~\ref{sec:quant-err-opt-rel-shift-pf}.

\progressonestep*

The proof crucially relies on the fact that $\grid/\gridstar \in \mathbb{Z}$, and is rooted in the following inequality:
\begin{lem}\label{lem:ineq} Let $y \in \mathbb{R}$ and $k\in\mathbb{Z}$. Then
\[
\left(1-\left\{ y\right\} \right)\left\{ y\right\} \leq k\left(1-\left\{ \frac{y}{k}\right\} \right)\left\{ \frac{y}{k}\right\} \,.
\]
\end{lem}
\begin{proof}
It suffices to consider $y \in [0,k]$, as both $\{y\}(1-\{y\})$ and $\{y/k\}(1-\{y/k\})$ are periodic within this interval. The function $\{y/k\}(1-\{y/k\})$ is a quadratic which is monotonically increasing over $[0,k/2]$ and symmetric around $k/2$. As $(1-\{y\})\{y\}$ is periodic on intervals of length $1$ it suffices to show that $(1-\{y\})\{y\} \leq k\left(1-\left\{ \frac{y}{k}\right\} \right)\left\{ \frac{y}{k}\right\}$ on the interval $[0,1]$. At this point we can drop the fractional part, and simply need to compare two quadratics over $[0,1]$.
Equivalently we need to show that
$k(1-y/k)y/k \geq y(1-y)$ over $[0,1]$, which after simplifying both sides is equivalent to $y^2(1-1/k) \geq 0$ over this interval, which is true.
\end{proof}

Finally, we provide some basic optimization inequalities, which will allow us to prove our theorems.

\paragraph{Optimization Basics.} The first Lemma bounds the change in function value using smoothness, while the latter upper bounds the $\ell_2$ distance to optimality using the error in function value. We provide the proofs in Sections~\ref{sec:progress-1-step-pf} and~\ref{sec:pl-norm-bd-pf}.
\begin{lem}
\label{lem:progress-1-step}Let $f:\mathbb{R}^{n}\rightarrow\mathbb{R}$
be a $\beta$-smooth function. Then for any $\vdelta\in\mathbb{R}^{n}$,
\[
f\left(\x+\vdelta\right)\leq f\left(\x\right)+\left(1-\eta\right)\left\langle \nabla f\left(\x\right),\vdelta\right\rangle -\frac{\eta^{2}}{2\beta}\left\Vert \nabla f\left(\x\right)\right\Vert _{2}^{2}+\frac{\beta}{2}\left\Vert \frac{\eta}{\beta}\nabla f\left(\x\right)+\vdelta\right\Vert _{2}^{2}\,.
\]
\end{lem}

\begin{lem}\label{lem:pl-norm-bd}
If $f:\mathbb{R}^{n}\rightarrow\mathbb{R}$ satisfies the $\alpha$-PL
condition, then for all $\x\in\mathbb{R}^{n}$,
\[
f\left(\x\right)-\fstar\geq\frac{\alpha}{2}\left\Vert \x-\xstar\right\Vert _{2}^{2}\,,
\]
where $\xstar\in\arg\min_{\x}f\left(\x\right)$.
\end{lem}

We are now ready to prove the main theorems in this paper.

\subsection{SGD with weight quantization}\label{sec:SGD-weight-quant}

We first prove the stepping lemma for our quantized gradient method, in the case where full gradients are available. The steps are essentially the same we described in Section~\ref{sec:background}.

\begin{lem}
\label{lem:step-determ}Let $f\left(\x\right):\mathbb{R}^{n}\rightarrow\mathbb{R}$
be a $\beta$-smooth and $\alpha$-PL function. For each $r\in\left[-\gridstar/2,\gridstar/2\right]$,
let $\xstar_{r,\gridstar}$ be any minimizer of $f$ over $\gridstar\mathbb{Z}+r$.
Let $\grid=\frac{\gridstar}{\left\lceil 4\left(\beta/\alpha\right)^{2}\right\rceil }$.
Then letting $\xp=\qs_{\grid}\left(\x-\frac{1}{\beta}\nabla f\left(\x\right)\right)$,
one has that in expectation over the random bits used by the quantization
operator:
\[
\mathbb{E}f\left(\xp\right)-\mathbb{E}f\left(\xstar_{r,\gridstar}\right)\leq\left(1-\frac{\alpha}{2\beta}\right)\left(\mathbb{E}f\left(\x\right)-\mathbb{E}f\left(\xstar_{r,\gridstar}\right)\right)\,.
\]
\end{lem}

\begin{proof}
Letting $\vdelta=\xp-\x$, we write:
\begin{align*}
\left\Vert \frac{1}{\beta}\nabla f\left(\x\right)+\vdelta\right\Vert _{2}^{2} & =\left\Vert \left(\x+\vdelta\right)-\left(\x-\frac{1}{\beta}\nabla f\left(\x\right)\right)\right\Vert _{2}^{2}=\left\Vert \qs_{\grid}\left(\x-\frac{1}{\beta}\nabla f\left(\x\right)\right)-\left(\x-\frac{1}{\beta}\nabla f\left(\x\right)\right)\right\Vert _{2}^{2}\,.
\end{align*}
Also using Lemma \ref{lem:quant-err-opt-rel-shift}, we have that
for any $\xstar\in\arg\min_{\x}f\left(\x\right),$
\begin{align*}
\mathbb{E}\left\Vert \qs_{\grid}\left(\x-\frac{1}{\beta}\nabla f\left(\x\right)\right)-\left(\x-\frac{1}{\beta}\nabla f\left(\x\right)\right)\right\Vert _{2}^{2} & \leq\frac{\grid}{\gridstar}\mathbb{E}_{r}\left[\left\Vert \x-\frac{1}{\beta}\nabla f\left(\x\right)-\xstar_{r,\gridstar}\right\Vert _{2}^{2}\right]\\
 & \leq2\frac{\grid}{\gridstar}\left(\mathbb{E}_{r}\left[\left\Vert \xstar_{r,\gridstar}-\xstar\right\Vert _{2}^{2}\right]+\left\Vert \x-\frac{1}{\beta}\nabla f\left(\x\right)-\xstar\right\Vert _{2}^{2}\right)\,.
\end{align*}
Using the $\alpha$-PL condition we upper bound distance from $\xstar$
with function value i.e.
\begin{align*}
\left\Vert \xstar_{r,\gridstar}-\xstar\right\Vert _{2}^{2} & \leq\frac{2}{\alpha}\cdot f\left(\xstar_{r,\gridstar}\right)-f\left(\xstar\right)\,,\\
\left\Vert \x-\frac{1}{\beta}\nabla f\left(\x\right)-\xstar\right\Vert _{2}^{2} & \leq\frac{2}{\alpha}\left(f\left(\x-\frac{1}{\beta}\nabla f\left(\x\right)-\xstar\right)-f\left(\xstar\right)\right)\,.
\end{align*}
Combining these with Lemma \ref{lem:progress-1-step} for $\eta=1$
we conclude that 
\begin{align*}
\mathbb{E}f\left(\xp\right) & \leq f\left(\x\right)-\frac{1}{2\beta}\left\Vert \nabla f\left(\x\right)\right\Vert _{2}^{2}\\
 & +2\frac{\grid}{\gridstar}\cdot\frac{\beta}{\alpha}\left(f\left(\x-\frac{1}{\beta}\nabla f\left(\x\right)\right)-f\left(\xstar\right)\right)+2\frac{\grid}{\gridstar}\cdot\frac{\beta}{\alpha}\cdot\mathbb{E}_{r}\left[f\left(\xstar_{r,\gridstar}\right)-f\left(\xstar\right)\right]\,.
\end{align*}
Again, using the PL condition we lower bound $\frac{1}{2}\left\Vert \nabla f\left(\x\right)\right\Vert _{2}^{2}\geq\alpha\left(f\left(\x\right)-f\left(\xstar\right)\right)$,
which gives
\begin{align*}
 & \mathbb{E}\left[f\left(\xp\right)-f\left(\xstar\right)\right]\\
 & \leq\left(1-\frac{\alpha}{\beta}\right)\left(f\left(\x\right)-f\left(\xstar\right)\right)+2\frac{\grid}{\gridstar}\cdot\frac{\beta}{\alpha}\left(f\left(\x-\frac{1}{\beta}\nabla f\left(\x\right)\right)-f\left(\xstar\right)\right)+2\frac{\grid}{\gridstar}\cdot\frac{\beta}{\alpha}\cdot\mathbb{E}_{r}\left[f\left(\xstar_{r,\gridstar}\right)-f\left(\xstar\right)\right]\\
 & \leq\left(1-\frac{\alpha}{\beta}+2\frac{\grid}{\gridstar}\cdot\frac{\beta}{\alpha}\right)\left(f\left(\x\right)-f\left(\xstar\right)\right)+2\frac{\grid}{\gridstar}\cdot\frac{\beta}{\alpha}\cdot\mathbb{E}_{r}\left[f\left(\xstar_{r,\gridstar}\right)-f\left(\xstar\right)\right]\,.
\end{align*}
Equivalently we obtain
\begin{align*}
\mathbb{E}f\left(\xp\right) &-\mathbb{E}f\left(\xstar_{r,\gridstar}\right)  \leq\left(1-\frac{\alpha}{\beta}+2\frac{\grid}{\gridstar}\cdot\frac{\beta}{\alpha}\right)\left(f\left(\x\right)-f\left(\xstar_{r,\gridstar}\right)\right)+2\frac{\grid}{\gridstar}\cdot\frac{\beta}{\alpha}\cdot\mathbb{E}\left[f\left(\xstar_{r,\gridstar}\right)-f\left(\xstar\right)\right]\\
 & +f\left(\xstar\right)-\mathbb{E}f\left(\xstar_{r,\gridstar}\right)\\
 & =\left(1-\frac{\alpha}{\beta}+2\frac{\grid}{\gridstar}\cdot\frac{\beta}{\alpha}\right)\left(f\left(\x\right)-\mathbb{E}f\left(\xstar_{r,\gridstar}\right)\right)+\left(2\frac{\grid}{\gridstar}\cdot\frac{\beta}{\alpha}-\frac{\alpha}{\beta}+2\frac{\grid}{\gridstar}\cdot\frac{\beta}{\alpha}\right)\cdot\mathbb{E}\left[f\left(\xstar_{r,\gridstar}\right)-f\left(\xstar\right)\right]\\
 & =\left(1-\frac{\alpha}{\beta}+2\frac{\grid}{\gridstar}\cdot\frac{\beta}{\alpha}\right)\left(f\left(\x\right)-\mathbb{E}f\left(\xstar_{r,\gridstar}\right)\right)+\left(4\frac{\grid}{\gridstar}\cdot\frac{\beta}{\alpha}-\frac{\alpha}{\beta}\right)\cdot\mathbb{E}\left[f\left(\xstar_{r,\gridstar}\right)-f\left(\xstar\right)\right]\,.
\end{align*}
Since we set $\grid/\gridstar=\frac{1}{\left\lceil 4\left(\beta/\alpha\right)^{2}\right\rceil }$,
the second term is non-positive. Therefore in this case we have
\[
\mathbb{E}f\left(\xp\right)-\mathbb{E}f\left(\xstar_{r,\gridstar}\right)\leq\left(1-\frac{\alpha}{2\beta}\right)\left(f\left(\x\right)-\mathbb{E}f\left(\xstar_{r,\gridstar}\right)\right)\,,
\]
which concludes the proof.
\end{proof}

We now generalize the proof of Lemma~\ref{lem:step-determ} to the case where only stochastic gradients are available. The proof is essentially the same, the main difference being that we isolate terms involving the difference between the stochastic and the true gradient, which we bound separately using our variance bound.

\begin{lem}
\label{lem:step-random}Let $f:\mathbb{R}^{n}\rightarrow\mathbb{R}$
be a $\beta$-smooth and $\alpha$-PL function. For each $r\in\left[-\gridstar/2,\gridstar/2\right]$,
let $\xstar_{r,\gridstar}$ be any minimizer of $f$ over $\gridstar\mathbb{Z}+r$.
Let $\grid=\frac{\eta}{\left\lceil 16\left(\beta/\alpha\right)^{2}\right\rceil }\cdot\gridstar$.
Let $\xp=\qs_{\grid}\left(\x-\frac{\eta}{\beta}\grad{\x}\right)$,
where $\grad{\x}$ is an unbiased estimator for $\nabla f\left(\x\right)$
i.e. $\mathbb{E}\left[\grad{\x}\vert\x\right]=\nabla f\left(\x\right)$,
and $0<\eta\leq1$ is a step size parameter. Furthermore assume that
the variance of $\grad{\x}$ is bounded $\mathbb{E}\left\Vert \grad{\x}-\nabla f\left(\x\right)\right\Vert _{2}^{2}\leq\sigma^{2}$,
for a real parameter $\sigma>0$. Then, for $r\sim\unif{\intgs}$,
in expectation over the gradient stochasticity:
\[
\mathbb{E}\left[f\left(\xp\right)\vert\x\right]-\mathbb{E}f\left(\xstar_{r,\gridstar}\right)\leq\left(1-\frac{3}{4}\eta\frac{\alpha}{\beta}\right)\left(f\left(\x\right)-\mathbb{E}f\left(\xstar_{r,\gridstar}\right)\right)+\frac{5}{4}\frac{\eta^{2}}{\beta}\sigma^{2}\,.
\]
\end{lem}

\begin{proof}
We follow the analysis from Lemma \ref{lem:step-determ}, while moving
the stochastic gradients into expressions that involve the stochastic
variance. Letting $\delta=\xp-\x$, we write:
\begin{align*}
\left\Vert \frac{\eta}{\beta}\nabla f\left(\x\right)+\vdelta\right\Vert _{2}^{2} & =\left\Vert \left(\x+\vdelta\right)-\left(\x-\frac{\eta}{\beta}\nabla f\left(\x\right)\right)\right\Vert _{2}^{2}=\left\Vert \qs_{r,\grid}\left(\x-\frac{\eta}{\beta}\grad{\x}\right)-\left(\x-\frac{\eta}{\beta}\nabla f\left(\x\right)\right)\right\Vert _{2}^{2}\\
 & \leq2\left\Vert \qs_{r,\grid}\left(\x-\frac{\eta}{\beta}\grad{\x}\right)-\left(\x-\frac{\eta}{\beta}\grad{\x}\right)\right\Vert _{2}^{2}+2\left\Vert \frac{\eta}{\beta}\left(\grad{\x}-\nabla f\left(\x\right)\right)\right\Vert _{2}^{2}\,,
\end{align*}
where we used the inequality $\left\Vert \va+\vb\right\Vert _{2}^{2}\leq2\left\Vert \va\right\Vert _{2}^{2}+2\left\Vert \vb\right\Vert _{2}^{2}$.
Also using Lemma \ref{lem:quant-err-opt-rel-shift}, we have that
for any $\xstar\in\arg\min_{\x}f\left(\x\right),$
\begin{align*}
 & \mathbb{E}\left\Vert \qs_{r,\grid}\left(\x-\frac{\eta}{\beta}\grad{\x}\right)-\left(\x-\frac{\eta}{\beta}\grad{\x}\right)\right\Vert _{2}^{2}\\
 & \leq\frac{\grid}{\gridstar}\mathbb{E}\left[\left\Vert \x-\frac{\eta}{\beta}\grad{\x}-\xstar_{r,\gridstar}\right\Vert _{2}^{2}\right]\\
 & \leq2\frac{\grid}{\gridstar}\left(\mathbb{E}\left[\left\Vert \xstar_{r,\gridstar}-\xstar\right\Vert _{2}^{2}\right]+\left\Vert \x-\frac{\eta}{\beta}\grad{\x}-\xstar\right\Vert _{2}^{2}\right)\\
 & \leq2\frac{\grid}{\gridstar}\left(\mathbb{E}\left[\left\Vert \xstar_{r,\gridstar}-\xstar\right\Vert _{2}^{2}\right]+2\left\Vert \x-\frac{\eta}{\beta}\nabla f\left(\x\right)-\xstar\right\Vert _{2}^{2}+2\left\Vert \frac{\eta}{\beta}\left(\grad{\x}-\nabla f\left(\x\right)\right)\right\Vert _{2}^{2}\right)\,.
\end{align*}
Using the $\alpha$-PL condition we upper bound distance from $\xstar$
with function value i.e.
\begin{align*}
\left\Vert \xstar_{r,\gridstar}-\xstar\right\Vert _{2}^{2} & \leq\frac{2}{\alpha}\cdot\left(f\left(\xstar_{r,\gridstar}\right)-f\left(\xstar\right)\right)\,,\\
\left\Vert \x-\frac{\eta}{\beta}\nabla f\left(\x\right)-\xstar\right\Vert _{2}^{2} & \leq\frac{2}{\alpha}\left(f\left(\x-\frac{\eta}{\beta}\nabla f\left(\x\right)\right)-f\left(\xstar\right)\right)\,.
\end{align*}
Combining these with Lemma \ref{lem:progress-1-step} we conclude
that in expectation over the random shift:
\begin{align*}
f\left(\xp\right) & \leq f\left(\x\right)+\left(1-\eta\right)\left\langle \nabla f\left(\x\right),\vdelta\right\rangle -\frac{\eta^{2}}{2\beta}\left\Vert \nabla f\left(\x\right)\right\Vert _{2}^{2}\\
 & +\frac{\beta}{2}\cdot\left(2\left\Vert \qs_{r,\grid}\left(\x-\frac{\eta}{\beta}\grad{\x}\right)-\left(\x-\frac{\eta}{\beta}\nabla f_{i}\left(\x\right)\right)\right\Vert _{2}^{2}+2\left\Vert \frac{\eta}{\beta}\left(\grad{\x}-\nabla f\left(\x\right)\right)\right\Vert _{2}^{2}\right)\\
 & \leq f\left(\x\right)+\left(1-\eta\right)\left\langle \nabla f\left(\x\right),\vdelta\right\rangle -\frac{\eta^{2}}{2\beta}\left\Vert \nabla f\left(\x\right)\right\Vert _{2}^{2}\\
 & +2\beta\frac{\grid}{\gridstar}\cdot\left(\left\Vert \xstar_{r,\gridstar}-\xstar\right\Vert _{2}^{2}+2\left\Vert x-\frac{\eta}{\beta}\nabla f\left(x\right)-\xstar\right\Vert _{2}^{2}+2\left\Vert \frac{\eta}{\beta}\left(\grad{\x}-\nabla f\left(\x\right)\right)\right\Vert _{2}^{2}\right)\\
 & +\frac{\eta^{2}}{\beta}\left\Vert \grad{\x}-\nabla f\left(\x\right)\right\Vert _{2}^{2}\\
 & \leq f\left(\x\right)+\left(1-\eta\right)\left\langle \nabla f\left(\x\right),\vdelta\right\rangle -\frac{\eta^{2}}{2\beta}\left\Vert \nabla f\left(\x\right)\right\Vert _{2}^{2}\\
 & +4\frac{\grid}{\gridstar}\frac{\beta}{\alpha}\left(f\left(\x-\frac{\eta}{\beta}\nabla f\left(\x\right)\right)-f\left(\xstar\right)\right)+4\frac{\grid}{\gridstar}\frac{\beta}{\alpha}\left(f\left(\xstar_{r,\gridstar}\right)-f\left(\xstar\right)\right)\\
 & +\frac{\eta^{2}}{\beta}\left(1+4\frac{\grid}{\gridstar}\right)\left\Vert \grad{\x}-\nabla f\left(\x\right)\right\Vert _{2}^{2}\,.
\end{align*}
Again, using the PL condition we lower bound $\frac{1}{2}\left\Vert \nabla f\left(\x\right)\right\Vert _{2}^{2}\geq\alpha\left(f\left(\x\right)-f\left(\xstar\right)\right)$,
which gives that in expectation over the random shift:
\begin{align*}
f\left(\xp\right)-f\left(\xstar\right) & \leq\left(1-\eta^{2}\frac{\alpha}{\beta}\right)\left(f\left(\x\right)-f\left(\xstar\right)\right)+\left(1-\eta\right)\left\langle \nabla f\left(\x\right),\vdelta\right\rangle \\
 & +4\frac{\grid}{\gridstar}\frac{\beta}{\alpha}\left(f\left(\x-\frac{\eta}{\beta}\nabla f\left(\x\right)\right)-f\left(\xstar\right)\right)+4\frac{\grid}{\gridstar}\frac{\beta}{\alpha}\left(f\left(\xstar_{r,\gridstar}\right)-f\left(\xstar\right)\right)\\
 & +\frac{\eta^{2}}{\beta}\left(1+4\frac{\grid}{\gridstar}\right)\left\Vert \grad{\x}-\nabla f\left(\x\right)\right\Vert _{2}^{2}\\
 & \leq\left(1-\eta^{2}\frac{\alpha}{\beta}+4\frac{\grid}{\gridstar}\frac{\beta}{\alpha}\right)\left(f\left(\x\right)-f\left(\xstar\right)\right)+4\frac{\grid}{\gridstar}\frac{\beta}{\alpha}\left(f\left(\xstar_{r,\gridstar}\right)-f\left(\xstar\right)\right)\\
 & +\left(1-\eta\right)\left\langle \nabla f\left(\x\right),\vdelta\right\rangle +\frac{\eta^{2}}{\beta}\left(1+4\frac{\grid}{\gridstar}\right)\left\Vert \grad{\x}-\nabla f\left(\x\right)\right\Vert _{2}^{2}\,.
\end{align*}
and equivalently, in expectation over the random shift:
\begin{align*}
\mathbb{E}\left[f\left(\xp\right)-f\left(\xstar_{r,\gridstar}\right)\right] & \leq\left(1-\eta^{2}\frac{\alpha}{\beta}+4\frac{\grid}{\gridstar}\frac{\beta}{\alpha}\right)\left(f\left(\x\right)-\mathbb{E}f\left(\xstar_{r,\gridstar}\right)\right)+\left(8\frac{\grid}{\gridstar}\frac{\beta}{\alpha}-\eta^{2}\frac{\alpha}{\beta}\right)\left(\mathbb{E}f\left(\xstar_{r,\gridstar}\right)-f\left(\xstar\right)\right)\\
 & +\left(1-\eta\right)\left\langle \nabla f\left(\x\right),\mathbb{E}\left[\vdelta\right]\right\rangle +\frac{\eta^{2}}{\beta}\left(1+4\frac{\grid}{\gridstar}\right)\left\Vert \grad{\x}-\nabla f\left(\x\right)\right\Vert _{2}^{2}\,.
\end{align*}
At this point we use Lemma \ref{lem:mean-var-T} to write 
\begin{align*}
\mathbb{E}\left[\vdelta\right] & =\mathbb{E}\left[\qs_{r,\grid}\left(\x-\frac{\eta}{\beta}\grad{\x}\right)-x\right]\\
 & =-\frac{\eta}{\beta}\grad{\x}\,,
\end{align*}
and thus 
\[
\left(1-\eta\right)\left\langle \nabla f\left(\x\right),\mathbb{E}\left[\vdelta\right]\right\rangle =-\frac{\eta}{\beta}\left(1-\eta\right)\left\langle \nabla f\left(\x\right),\grad{\x}\right\rangle \,.
\]
Therefore, after taking expectation over both the random shift and
gradient stochasticity we obtain: 
\begin{align*}
 & \mathbb{E}\left[f\left(\xp\right)-f\left(\xstar_{r,\gridstar}\right)\vert\x\right]\\
 & \leq\left(1-\eta^{2}\frac{\alpha}{\beta}+4\frac{\grid}{\gridstar}\frac{\beta}{\alpha}\right)\left(f\left(\x\right)-\mathbb{E}f\left(\xstar_{r,\gridstar}\right)\right)+\left(8\frac{\grid}{\gridstar}\frac{\beta}{\alpha}-\eta^{2}\frac{\alpha}{\beta}\right)\left(\mathbb{E}f\left(\xstar_{r,\gridstar}\right)-f\left(\xstar\right)\right)\\
 & -\frac{\eta}{\beta}\left(1-\eta\right)\left\Vert \nabla f\left(\x\right)\right\Vert _{2}^{2}+\frac{\eta^{2}}{\beta}\left(1+4\frac{\grid}{\gridstar}\right)\sigma^{2}\\
 & \leq\left(1-\eta^{2}\frac{\alpha}{\beta}+4\frac{\grid}{\gridstar}\frac{\beta}{\alpha}\right)\left(f\left(\x\right)-\mathbb{E}f\left(\xstar_{r,\gridstar}\right)\right)+\left(8\frac{\grid}{\gridstar}\frac{\beta}{\alpha}-\eta^{2}\frac{\alpha}{\beta}\right)\left(\mathbb{E}f\left(\xstar_{r,\gridstar}\right)-f\left(\xstar\right)\right)\\
 & -2\eta\left(1-\eta\right)\frac{\alpha}{\beta}\left(f\left(\x\right)-f\left(\xstar\right)\right)+\frac{\eta^{2}}{\beta}\left(1+4\frac{\grid}{\gridstar}\right)\sigma^{2}\\
 & =\left(1-\eta^{2}\frac{\alpha}{\beta}-2\eta\left(1-\eta\right)\frac{\alpha}{\beta}+4\frac{\grid}{\gridstar}\frac{\beta}{\alpha}\right)\left(f\left(\x\right)-\mathbb{E}f\left(\xstar_{r,\gridstar}\right)\right)+\left(8\frac{\grid}{\gridstar}\frac{\beta}{\alpha}-\eta^{2}\frac{\alpha}{\beta}-2\eta\left(1-\eta\right)\frac{\alpha}{\beta}\right)\left(\mathbb{E}f\left(\xstar_{r,\gridstar}\right)-f\left(\xstar\right)\right)\\
 & +\frac{\eta^{2}}{\beta}\left(1+4\frac{\grid}{\gridstar}\right)\sigma^{2}\,.
\end{align*}
Since we set $\grid/\gridstar=\frac{\eta}{\left\lceil 16\left(\beta/\alpha\right)^{2}\right\rceil }$,
the second term is non-positive. Therefore in this case we have
\begin{align*}
\mathbb{E}\left[f\left(\xp\right)-\mathbb{E}f\left(\xstar_{r,\gridstar}\right)\vert\x\right] & \leq\left(1-\eta^{2}\frac{\alpha}{\beta}-2\eta\left(1-\eta\right)\frac{\alpha}{\beta}+\frac{\eta}{4}\frac{\alpha}{\beta}\right)\left(f\left(\x\right)-\mathbb{E}f\left(\xstar_{r,\gridstar}\right)\right)+\frac{\eta^{2}}{\beta}\left(1+\frac{\eta}{4}\left(\frac{\alpha}{\beta}\right)^{2}\right)\sigma^{2}\\
 & \leq\left(1-\frac{7}{4}\eta\frac{\alpha}{\beta}+\eta^{2}\frac{\alpha}{\beta}\right)\left(f\left(\x\right)-\mathbb{E}f\left(\xstar_{r,\gridstar}\right)\right)+\frac{\eta^{2}}{\beta}\left(1+\frac{\eta}{4}\left(\frac{\alpha}{\beta}\right)^{2}\right)\sigma^{2}\\
 & \leq\left(1-\frac{3}{4}\eta\frac{\alpha}{\beta}\right)\left(f\left(\x\right)-\mathbb{E}f\left(\xstar_{r,\gridstar}\right)\right)+\frac{5}{4}\frac{\eta^{2}}{\beta}\sigma^{2}\,,
\end{align*}
as long as $\eta\leq1$. This concludes the proof.
\end{proof}

Using Lemma~\ref{lem:step-random} the proof of Theorem~\ref{thm:mainthm} follows very easily.

\mainthm*

\begin{proof}
Plugging in Lemma \ref{lem:step-random} and applying it for $T=\frac{10}{\eta}\frac{\beta}{\alpha}\ln\frac{f\left(\x_{0}\right)-\mathbb{E}f\left(\xstar_{r,\gridstar}\right)}{\err}$
we obtain:
\begin{align*}
\mathbb{E}f\left(\x_{T}\right)-\mathbb{E}f\left(\xstar_{r,\gridstar}\right) & \leq\frac{\err}{2}+\frac{5}{4}\frac{\eta^{2}}{\beta}\sigma^{2}\cdot\sum_{k=0}^{T-1}\left(1-\frac{3}{4}\eta\frac{\alpha}{\beta}\right)^{k}
  \leq\frac{\err}{2}+\frac{5}{4}\frac{\eta^{2}}{\beta}\sigma^{2}\cdot\frac{4}{3}\frac{1}{\eta}\frac{\beta}{\alpha}
  =\frac{\err}{2}+\frac{5}{3}\frac{\eta}{\alpha}\sigma^{2}\,.
\end{align*}
Since we set $\eta=\min\left\{ \frac{3}{10}\frac{\err\alpha}{\sigma^{2}},1\right\} $,
the entire quantity is at most $\err$, which concludes the proof.
\end{proof}

\subsection{Reducing Communication by Quantizing Gradients}\label{sec:quantgrad}

The approach described in Section \ref{sec:SGD-weight-quant} maintains
quantized weights, but communicating gradients may still be expensive.
In this section we show that any reasonable quantization method for
gradients can be used to reduce communication, while paying in exchange
an increased variance. This trade-off is inherent, as the
reduction in the number of bits requires injecting randomness, so
as the entropy of the output is not smaller than that of the original
message to be communicated.

To do so we use any gradient quantization method $\qf$, as long as it is an unbiased estimator for the input it takes, and has bounded variance. Our formal requirements  for $\qf$ are the following.

\begin{defn}\label{def:gradquantprop}
We say that a gradient quantization operator $\qf$ is a $(\sigmagrad, b)$-unbiased quantizer if it:
\begin{enumerate}
\item is an unbiased estimator: $\mathbb{E}\left[ \qf(\grad{\x}) \vert \grad{\x}\right] = \grad{\x}$,
\item has bounded variance on the stochastic gradients: $\mathbb{E}\left[ \left\Vert\qf(\grad{\x})  - \grad{\x} \right\Vert_2^2 | \grad{\x} \right] \leq \sigmagrad^2$, 
\item requires $b$ bits to communicate $\qf(\grad{\x})$.
\end{enumerate}
\end{defn}

By Lemma~\ref{lem:mean-var-T}, these requirements are automatically satisfied by our shift-and-round quantization operator $\qs$, and we can show that $\sigmagrad$ and $b$ are determined by the $\ell_1$ norm of $\grad{\x}$.

\paragraph{Standard Quantization Schemes and Their Communication Cost.} Another standard gradient quantization scheme can be obtained by independently rounding each coordinate to one of the neighboring points on the quantization grid, with an appropriate probability. An identical scheme has been previously used in other related works on gradient quantization~\cite{alistarh2017qsgd}.

\begin{defn}
[quantization by flipping a coin] \label{def:qflip}Let $\grid>0$
be a scalar defining the coarseness of the quantization grid. Let
the operator $\qg_{\grid}:\mathbb{R}\rightarrow\grid\mathbb{Z}$ defined
as
\[
\qg_{\grid}\left(x\right)=\begin{cases}
\grid\left\lfloor \frac{x}{\grid}\right\rfloor  & \text{with probability }1-\left(\frac{x}{\grid}-\left\lfloor \frac{x}{\grid}\right\rfloor \right)\\
\grid\left(\left\lfloor \frac{x}{\grid}\right\rfloor +1\right) & \text{with probability }\frac{x}{\grid}-\left\lfloor \frac{x}{\grid}\right\rfloor 
\end{cases}
\]
where $\grid>0$. We apply $\qg_{\grid}$ to vectors, with the meaning
that it is independently applied to each coordinate.
\end{defn}
It is fairly easy to prove that this satisfies very similar properties to those proved for $\qs$ in Lemma~\ref{lem:mean-var-T}, which we quickly prove in Section~\ref{sec:mean-var-Q-pf}.  We notice an important difference between these two quantization methods. While $\qg$ independently quantizes each coordinate, the quantization in $\qs$ is done dependently across coordinates, and the output is always a vector in $\grid \mathbb{Z}^n + r\one$, for a randomly sampled scalar $r$. Although morally they are quite similar (in fact, the shift after rounding could just as well be ignored, and still have an unbiased extimator), it is important if we want to relate the quality of the final solution to the best set of weights from a reasonably chosen grid. This difference is apparent when trying to provide bounds of the type of Lemma~\ref{lem:quant-err-opt-rel-shift}, bu this attempt falls through in the case of the $\qg$ operator.

As we can naively relate the  communication cost of a quantized gradient to its sparsity, it is important to discuss quantitative bounds. In both cases, the sparsity bound depends on the $\ell_{1}$ norm of the quantized
vector, and its easy to see that it is tight. By comparison, the bound
from~\cite{alistarh2017qsgd} is provided in terms of the $\ell_{2}$ norm of the vector, but
pays an additional $\sqrt{n}$ factor, which is suboptimal when the input is analytically sparse. For $\qs$ and $\qg$, we see that the variance introduced by quantizing
a generic vector $\v$ is bounded by $\grid\left\Vert \v\right\Vert _{1}$,
while its sparsity is $\left\Vert \v\right\Vert _{1}/\grid$. Hence a naive encoding of this quantized gradient requires at most $O\left(\frac{\left\Vert \v\right\Vert _{1}}{\grid}\left(\ln n+\ln{\left\Vert \v\right\Vert _{1}}\right)\right)$
bits of communication.

\paragraph{SGD with Weight and Gradient Quantization.} For gradient quantization operators that are unbiased estimators, we can use them as stochastic gradients inside the scheme we derived in Theorem
\ref{thm:mainthm}. To do so we crucially use the following identity involving conditional variance:
\begin{lem}
[Law of total variance]\label{lem:law-tot-var} Given random variables
$X$ and $Y$, one has that
\[
\text{Var}\left[Y\right]=\mathbb{E}\left[\text{Var}\left[Y\vert X\right]\right]+\text{Var}\left[\mathbb{E}\left[Y\vert X\right]\right]\,.
\]
\end{lem}

\begin{cor}
\label{cor:quant-var}Consider a stochastic gradient estimator $\grad{\x}$
such that $\mathbb{E}\left[\grad{\x}\vert\x\right]=\nabla f\left(\x\right)$
and $\mathbb{E}\left[\left\Vert \grad{\x}-\nabla f\left(\x\right)\right\Vert _{2}^{2} \vert \x \right]\leq\sigma^{2}$.
Consider   $(\sigmagrad,b)$-unbiased quantizer (Definition~\ref{def:gradquantprop}).
Then 
\[
\mathbb{E}\left[\qg_{\grid}\left(\grad{\x}\right)\vert\x\right]=\nabla f\left(\x\right)\,,
\]
i.e. it is an unbiased estimator for the gradient, and
\[
\mathbb{E}\left[\left\Vert \qg_{\grid}\left(\grad{\x}\right)-\nabla f\left(\x\right)\right\Vert _{2}^{2}\right]
\leq
\sigmagrad^2+\sigma^{2}\,.
\]
\end{cor}

\begin{proof}
The fact that the quantized gradient is an unbiased estimator for
$\nabla f\left(\x\right)$ follows from the law of total expectation,
as 
\begin{align*}
\mathbb{E}\left[\qg_{\grid}\left(\grad{\x}\right)\vert\x\right] & =\mathbb{E}\left[\mathbb{E}\left[\qg_{\grid}\left(\grad{\x}\right)\vert\x,\grad{\x}\right]\right]=\mathbb{E}\left[\mathbb{E}\left[\grad{\x}\vert\x\right]\right]=\nabla f\left(\x\right)\,.
\end{align*}
For the variance, we use Lemma \ref{lem:law-tot-var} to write: 
\begin{align*}
\mathbb{E}\left[\left\Vert \qg_{\grid}\left(\grad{\x}\right)-\nabla f\left(\x\right)\right\Vert _{2}^{2}\right] & =\text{Var}\left[\qg_{\grid}\left(\grad{\x}\right)\right]\\
 & =\mathbb{E}\left[\text{Var}\left[\qg_{\grid}\left(\grad{\x}\right)\vert\grad{\x}\right]\right]+\text{Var}\left[\mathbb{E}\left[\qg_{\grid}\left(\grad{\x}\right)\vert\grad{\x}\right]\right]\\
 & \leq
 \sigmagrad^2
 +\text{Var}\left[\grad{\x}\right]\\
 & =\sigmagrad^2+\sigma^{2}\,.
\end{align*}
\end{proof}
Finally, combining Theorem \ref{thm:mainthm} with Corollary \ref{cor:quant-var}, we obtain
the final result from Corollary~\ref{cor:qwqg}.

\subsection{Deferred Proofs}\label{sec:deferredproofs}

\subsubsection{Proof of Lemma~\ref{lem:mean-var-T}}\label{sec:mean-var-T-pf}
\begin{proof}
For both the mean and variance computation, it suffices to prove these bounds for the scalar operator.

We note that by definition $\qs_{r,\grid}\left(x\right)-r=\qs_{0,\grid}\left(x-r\right)=\grid\cdot\qs_{0,1}\left(\frac{x-r}{\grid}\right):=\grid\cdot\left\lfloor \frac{x-r}{\grid}\right\rceil $.
Also let $\left\{ x\right\} =x-\lfloor x\rfloor$ denote the fractional
part of $x$. We can easily verify that for any scalar $0\leq z<1$,
we have 
\begin{equation}
\mathbb{E}_{u\sim\unif{\interval}}\left[\left\lfloor z+u\right\rceil \right]=z\,.\label{eq:shift-and-round}
\end{equation}
This is because $\left\lfloor z+u\right\rceil =1$ if and only if
$z+u\geq1/2$ i.e. $u\geq\frac{1}{2}-z$, which happens with probability
$z$. Now we can express the expectation of $\qs_{r,\epsilon}\left(x\right)$
as follows:

\begin{align*}
\mathbb{E}\left[\qs_{\grid}\left(x\right)\right] & =\mathbb{E}_{r}\left[\qs_{0,\grid}\left(x-r\right)+r\right]\\
 & =\mathbb{E}_{r}\left[\qs_{0,\grid}\left(x-r\right)\right]+\mathbb{E}_{r}\left[r\right]\\
 & =\mathbb{E}_{r}\left[\qs_{0,\grid}\left(\grid\left\lfloor \frac{x}{\grid}\right\rfloor +\grid\left\{ \frac{x}{\grid}\right\} -r\right)\right]+\mathbb{E}_{r}\left[r\right]\\
 & =\grid\left\lfloor \frac{x}{\grid}\right\rfloor +\mathbb{E}_{r}\left[\qs_{0,\epsilon}\left(\grid\left\{ \frac{x}{\grid}\right\} -r\right)\right]+\mathbb{E}_{r}\left[r\right]\\
 & =\grid\left\lfloor \frac{x}{\grid}\right\rfloor +\mathbb{E}_{r}\left[\grid\cdot\qs_{0,1}\left(\left\{ \frac{x}{\grid}\right\} -\frac{r}{\grid}\right)\right]+\mathbb{E}_{r}\left[r\right]\,.
\end{align*}
In the last line we used the fact that $\qs_{0,\grid}\left(y\right)=\grid\cdot\qs_{0,1}\left(\frac{y}{\grid}\right)$.
Now we reparameterize by using $u:=r/\grid$, where $r\sim\unif{\interval}$.
This allows to write the term in the middle as:
\[
\mathbb{E}_{r}\left[\grid\cdot\qs_{0,1}\left(\left\{ \frac{x}{\grid}\right\} -\frac{r}{\grid}\right)\right]=\grid\cdot\mathbb{E}_{u\sim\unif{\interval}}\left[\qs_{0,1}\left(\left\{ \frac{x}{\grid}\right\} -u\right)\right]=\grid\cdot\left\{ \frac{x}{\grid}\right\} \,,
\]
were we used (\ref{eq:shift-and-round}). Plugging back in we obtain
that
\[
\mathbb{E}\left[\qs_{r,\grid}\left(x\right)\right]=\grid\left\lfloor \frac{x}{\grid}\right\rfloor +\grid\cdot\left\{ \frac{x}{\grid}\right\} +0=x\,.
\]
Next we compute the scalar variance:
\begin{align*}
\mathbb{E}\left[\left(\qs_{\grid}\left(x\right)-x\right)^{2}\right] & =\mathbb{E}_{r}\left[\left(\qs_{0,\grid}\left(x-r\right)-x\right)^{2}\right]\\
 & =\mathbb{E}_{r}\left[\left(\grid\left\lfloor \frac{x}{\grid}\right\rfloor +\grid\cdot\qs_{0,1}\left(\left\{ \frac{x}{\grid}\right\} -\frac{r}{\grid}\right)-x\right)^{2}\right]\\
 & =\mathbb{E}_{r}\left[\grid^{2}\cdot\left(\left\lfloor \frac{x}{\grid}\right\rfloor +\qs_{0,1}\left(\left\{ \frac{x}{\grid}\right\} -\frac{r}{\grid}\right)-\frac{x}{\grid}\right)^{2}\right]\\
 & =\mathbb{E}_{u\sim\unif{\interval}}\left[\grid^{2}\cdot\left(\left\lfloor \frac{x}{\grid}\right\rfloor +\qs_{0,1}\left(\left\{ \frac{x}{\grid}\right\} -u\right)-\frac{x}{\grid}\right)^{2}\right]\\
 & =\grid^{2}\cdot\mathbb{E}_{u\sim\unif{\interval}}\left[\left(\qs_{0,1}\left(\left\{ \frac{x}{\grid}\right\} -u\right)-\left\{ \frac{x}{\grid}\right\} \right)^{2}\right]\,.
\end{align*}
Now we use the fact that for any scalar $0\leq z<1$ one has that
\[
\mathbb{E}_{u\sim\unif{\interval}}\left[\left(\left\lfloor z+u\right\rceil -z\right)^{2}\right]=z^{2}\,.
\]
This follows from the fact that $\left\lfloor z+u\right\rceil =1$
iff $u\geq1/2-z$, which happens with probability $z$, and makes
the expectation equal to
\[
\int_{-1/2}^{1/2-z}z^{2}du+\int_{1/2-z}^{1/2}\left(1-z\right)^{2}du=z\left(1-z\right)\,,
\]
which leads us to 
\[
\mathbb{E}\left[\left(\qs_{\grid}\left(x\right)-x\right)^{2}\right]=\grid^{2}\cdot\left\{ \frac{x}{\grid}\right\} \left(1-\left\{ \frac{x}{\grid}\right\} \right)\,,
\]
which gives us what we needed.

Finally, for the sparsity bound,  let us understand when a single scalar gets rounded to zero (before shifting back by $r$). We have that for $x\in\mathbb{R}$,
\begin{align*}
\mathbb{P}&\left[\qs_{r,\grid}\left(x\right)-r=0\right] 
 =\mathbb{P}\left[\qs_{r,1}\left(\frac{x}{\grid}\right)-r=0\right]
  =\begin{cases}
\int_{-1/2}^{1/2}\mathbf{1}_{\left\lfloor \frac{x}{\grid}-r\right\rceil =0}dr, & \left|x\right|<\grid,\\
0, & \grid\leq\left|x\right|,
\end{cases}\\
 & =\begin{cases}
\int_{-1/2}^{1/2}\mathbf{1}_{-\frac{1}{2}\leq\frac{x}{\grid}-r\leq\frac{1}{2}}dr, & \left|x\right|<\grid,\\
0, & \grid\leq\left|x\right|,
\end{cases}
  =\begin{cases}
\int_{-1/2}^{1/2}\mathbf{1}_{\frac{x}{\grid}-\frac{1}{2}\leq r\leq\frac{x}{\grid}+\frac{1}{2}}dr, & \left|\frac{x}{\grid}\right|<1,\\
0, & 1\leq\left|\frac{x}{\grid}\right|,
\end{cases}\\
 & =\min\left\{ \frac{x}{\grid}+\frac{1}{2},\frac{1}{2}\right\} -\max\left\{ \frac{x}{\grid}-\frac{1}{2},-\frac{1}{2}\right\} 
 =\frac{1}{2}+\min\left\{ \frac{x}{\grid},0\right\} -\left(-\frac{1}{2}+\max\left\{ \frac{x}{\grid},0\right\} \right)\\
 & =1+\min\left\{ \frac{x}{\grid},0\right\} -\max\left\{ \frac{x}{\grid},0\right\} 
  =1-\left|\frac{x}{\grid}\right|\,.
 \end{align*}
which shows that 
\begin{align*}
\mathbb{E}\left[\left\Vert \qf_{\grid}\left(\v\right)\right\Vert _{0}\right] & =\sum_{i=1}^{n}\left(1-\mathbb{P}\left[\qf_{\grid}\left(v_{i}\right)=0\right]\right)
  =\sum_{i=1}^{n}\begin{cases}
\left|\frac{v_{i}}{\grid}\right|, & \left|v_{i}\right|<\grid,\\
1, & \grid\leq\left|v_{i}\right|,
\end{cases}
  \leq\left\Vert \v\right\Vert _{1}/\grid\,.
\end{align*}

This concludes the proof.
\end{proof}

\subsubsection{Proof of Lemma~\ref{lem:quant-err-opt-rel-shift}}\label{sec:quant-err-opt-rel-shift-pf}
\begin{proof}
It suffices to prove this coordinate-wise. From Lemma \ref{lem:mean-var-T}
we have that for any $x\in\mathbb{R}$,
\[
\mathbb{E}\left[\left(\qs_{\grid}\left(x\right)-x\right)^{2}\right]=\grid^{2}\left(1-\left\{ \frac{x}{\grid}\right\} \right)\left\{ \frac{x}{\grid}\right\} 
\]
and similarly for $\gridstar$. Let $k=\gridstar/\grid$. Then 
\[
\mathbb{E}\left[\left(\qs_{\gridstar}\left(x\right)-x\right)^{2}\right]=k^{2}\grid^{2}\left(1-\left\{ \frac{x/\grid}{k}\right\} \right)\left\{ \frac{x/\grid}{k}\right\} 
\]
Applying the inequality from Lemma~\ref{lem:ineq}, 
we conclude that 
\[
\mathbb{E}\left[\left(\qs_{\grid}\left(x\right)-x\right)^{2}\right]=\grid^{2}\left(1-\left\{ \frac{x}{\grid}\right\} \right)\left\{ \frac{x}{\grid}\right\} \leq\grid^{2}\cdot k\left(1-\left\{ \frac{x}{k\grid}\right\} \right)\left\{ \frac{x}{k\grid}\right\} =\frac{1}{k}\mathbb{E}\left[\left(\qs_{\gridstar}\left(x\right)-x\right)^{2}\right]\,.
\]
Applying this bound to all coordinates we obtain
\[
\mathbb{E}\left[\left\Vert \qs_{\grid}\left(\x\right)-\x\right\Vert _{2}^{2}\right]\leq\frac{\grid}{\gridstar}\mathbb{E}\left[\left\Vert \qs_{r,\gridstar}\left(\x\right)-\x\right\Vert _{2}^{2}\right]\,.
\]
Also since $\qs_{r,\gridstar}$ rounds to the nearest point in $\gridstar\mathbb{Z}+r$,
clearly $\left\Vert \qs_{r,\gridstar}\left(\x\right)-\x\right\Vert _{2}^{2}\leq\left\Vert \xstar_{r,\gridstar}-\x\right\Vert _{2}^{2}$
for all $r$. Taking expectations on both sides and combining with the previous inequality concludes the proof.
\end{proof}

\subsubsection{Proof of Lemma~\ref{lem:progress-1-step}}\label{sec:progress-1-step-pf}
\begin{proof}
Using smoothness we have 
\begin{align*}
f\left(\x+\vdelta\right) & \leq f\left(\x\right)+\left\langle \nabla f\left(\x\right),\vdelta\right\rangle +\frac{\beta}{2}\left\Vert \vdelta\right\Vert _{2}^{2}\\
 & =f\left(\x\right)+\left(1-\eta\right)\left\langle \nabla f\left(\x\right),\vdelta\right\rangle -\frac{\eta^{2}}{2\beta}\left\Vert \nabla f\left(\x\right)\right\Vert _{2}^{2}+\left(\frac{\eta^{2}}{2\beta}\left\Vert \nabla f\left(\x\right)\right\Vert _{2}^{2}+\eta\left\langle \nabla f\left(\x\right),\vdelta\right\rangle +\frac{\beta}{2}\left\Vert \vdelta\right\Vert _{2}^{2}\right)\\
 & =f\left(\x\right)+\left(1-\eta\right)\left\langle \nabla f\left(\x\right),\vdelta\right\rangle -\frac{\eta^{2}}{2\beta}\left\Vert \nabla f\left(\x\right)\right\Vert _{2}^{2}+\frac{\beta}{2}\left\Vert \frac{\eta}{\beta}\nabla f\left(\x\right)+\vdelta\right\Vert _{2}^{2}\,.
\end{align*}
\end{proof}

\subsubsection{Proof of Lemma~\ref{lem:pl-norm-bd}}
\label{sec:pl-norm-bd-pf}
The proof is standard and can be found in literature, such as~\cite{karimi2016linear}. However, for completeness we reproduce it here.
\begin{proof}
 Let $g\left(\xx\right)=\sqrt{f\left(\xx\right)-f^{*}}$
for which we have
\[
\nabla g\left(\xx\right)=\frac{1}{2\sqrt{f\left(\xx\right)-f^{*}}}\nabla f\left(\xx\right)\ .
\]
Using the $\alpha$-PL condition we have
\begin{align*}
\left\Vert \nabla g\left(\xx\right)\right\Vert ^{2} & =\frac{1}{4\left(f\left(\xx\right)-f^{*}\right)}\cdot\left\Vert \nabla f\left(\xx\right)\right\Vert ^{2}
\geq\frac{1}{2\left(f\left(\xx\right)-f^{*}\right)}\cdot{\alpha}\cdot\left(f\left(\xx\right)-f^{*}\right)
=\frac{\alpha}{2}\ .
\end{align*}
Now starting at some $\xx_{0}$, we consider the dynamic $\dot{\xx}=-\nabla g\left(\xx\right)$.
We see that this always decreases function value until it reaches
some $\xx_{T}$ for which $\nabla g\left(\xx_{T}\right)=0$
and hence by the PL inequality, $\xx_{T}$ is a minimizer
i.e. $f\left(\xx_{T}\right)=f^{*}$. Now we can write 
\begin{align*}
g\left(\xx_{T}\right) & =g\left(\xx_{0}\right)+\int_{0}^{T}\left\langle \nabla g\left(\xx_{t}\right),\dot{\xx_{t}}\right\rangle dt=g\left(\xx_{0}\right)+\int_{0}^{T}\left\langle \nabla g\left(\xx_{t}\right),-\nabla g\left(\xx_{t}\right)\right\rangle dt\\
 & =g\left(\xx_{0}\right)-\int_{0}^{T}\left\Vert \nabla g\left(\xx_{t}\right)\right\Vert ^{2}dt\ .
\end{align*}
Thus 
\begin{align*}
g\left(\xx_{0}\right)-g\left(\xx_{T}\right) & =\int_{0}^{T}\left\Vert \nabla g\left(\xx_{t}\right)\right\Vert ^{2}dt\geq\sqrt{\frac{\alpha}{2}}\cdot\int_{0}^{T}\left\Vert \nabla g\left(\xx_{t}\right)\right\Vert dt=\sqrt{\frac{\alpha}{2}}\cdot\int_{0}^{T}\left\Vert \dot{\xx_{t}}\right\Vert dt\ ,
\end{align*}
where we used our lower bound on the norm of $\nabla g\left(\xx\right)$.
Finally, we use the fact that the last integral lower bounds the total
movement of $\xx$ as it moves from $\xx_{0}$ to $\xx_{T}$. Thus
\[
\int_{0}^{T}\left\Vert \dot{\xx_{t}}\right\Vert dt\geq\left\Vert \xx_{0}-\xx_{T}\right\Vert \ ,
\]
so 
\[
g\left(\xx_{0}\right)-g\left(\xx_{T}\right)\geq\sqrt{\frac{\alpha}{2}}\left\Vert \xx_{0}-\xx_{T}\right\Vert \ ,
\]
which enables us to conclude that
\[
f\left(\xx_{0}\right)-f^{*}\geq\frac{\alpha}{2}\left\Vert \xx_{0}-\xx_{T}\right\Vert ^{2}\ ,
\]
where $\xx_{T}$ is some global minimizer of $f$. This concludes
the proof.
\end{proof}

\subsubsection{Bound for Quantization by Coin Flip}
\label{sec:mean-var-Q-pf}
\begin{lem}
\label{lem:mean-var-Q}Let $\v\in\mathbb{R}^n$, and let $\grid>0$, and let $\qg_\grid$ be the quantization operator from Definition~\ref{def:qflip}.
Then,
\begin{align*}
\mathbb{E}\left[\qg_{\grid}\left(\v\right)\right] & =\v\,,\\
\mathbb{E}\left[\left\Vert \qg_{\grid}\left(\v\right)-\v\right\Vert_2^{2}\right] & =\grid^{2}\cdot \sum_{i=1}^n \left\{ \frac{v_i}{\grid}\right\} \left(1-\left\{ \frac{v_i}{\grid}\right\} \right)\,,\\
\mathbb{E}\left[\left\Vert \qg_{\grid}\left(\v\right)\right\Vert _{0}\right] & \leq\left\Vert \v\right\Vert _{1}/\grid\,.
\end{align*}
\end{lem}
\begin{proof}
For the expectation and variance, it suffices to prove that these bound holds coordinate-wise. Let $x\in\mathbb{R}$, and write $x=\grid\left(\left\lfloor \frac{x}{\grid}\right\rfloor +\left\{ \frac{x}{\grid}\right\} \right)$
so that 
\begin{align*}
\mathbb{E}\left[\qg_{\grid}\left(x\right)\right] & =\mathbb{E}\left[\qg_{\grid}\left(\grid\left(\left\lfloor \frac{x}{\grid}\right\rfloor +\left\{ \frac{x}{\grid}\right\} \right)\right)\right]\\
 & =\grid\left\lfloor \frac{x}{\grid}\right\rfloor +\mathbb{E}\left[\qg_{\grid}\left(\grid\left\{ \frac{x}{\grid}\right\} \right)\right]\\
 & =\grid\left\lfloor \frac{x}{\grid}\right\rfloor +\mathbb{E}\left[\qg_{\grid}\left(\grid\left\{ \frac{x}{\grid}\right\} \right)\right]\\
 & =\grid\left\lfloor \frac{x}{\grid}\right\rfloor +\grid\cdot\left\{ \frac{x}{\grid}\right\} \\
 & =x\,.
\end{align*}
Similarly we write the variance as:
\begin{align*}
\mathbb{E}\left[\left(\qg_{\grid}\left(x\right)-x\right)^{2}\right] & =\mathbb{E}\left[\left(\qg_{\grid}\left(\grid\left\{ \frac{x}{\grid}\right\} \right)-\grid\left\{ \frac{x}{\grid}\right\} \right)^{2}\right]\\
 & =\left(1-\left\{ \frac{x}{\grid}\right\} \right)\left(\grid\left\{ \frac{x}{\grid}\right\} \right)^{2}+\left\{ \frac{x}{\grid}\right\} \cdot\left(\grid-\grid\left\{ \frac{x}{\grid}\right\} \right)^{2}\\
 & =\grid^{2}\left(\left(1-\left\{ \frac{x}{\grid}\right\} \right)\left\{ \frac{x}{\grid}\right\} ^{2}+\left\{ \frac{x}{\grid}\right\} \cdot\left(1-\left\{ \frac{x}{\grid}\right\} \right)^{2}\right)\\
 & =\grid^{2}\left(1-\left\{ \frac{x}{\grid}\right\} \right)\left\{ \frac{x}{\grid}\right\} \,,
\end{align*}
For the sparsity bound, we need to understand when a single scalar gets rounded to zero. We have that for $x\in\mathbb{R}$,
\[
\mathbb{P}\left[\qg_{\grid}\left(x\right)=0\right]=\begin{cases}
1-\left|\frac{x}{\grid}\right|, & \left|x\right|<\grid,\\
0, & \grid\leq\left|x\right|,
\end{cases}
\]
which shows that 
\begin{align*}
\mathbb{E}\left[\left\Vert \qg_{\grid}\left(\v\right)\right\Vert _{0}\right] & =\sum_{i=1}^{n}\left(1-\mathbb{P}\left[\qg_{\grid}\left(v_{i}\right)=0\right]\right)\\
 & =\sum_{i=1}^{n}\begin{cases}
\left|\frac{v_{i}}{\grid}\right|, & \left|v_{i}\right|<\grid,\\
1, & \grid\leq\left|v_{i}\right|,
\end{cases}\\
 & \leq\left\Vert \v\right\Vert _{1}/\grid\,.
\end{align*}
\end{proof}

\end{document}